\documentclass[a4paper]{amsart} 

\usepackage{amsmath,amsthm,amssymb,amsfonts,mathrsfs,color,hyperref, mathtools,crop, graphicx, enumitem, todonotes}
\usepackage[]{geometry}

\theoremstyle{plain}
\begingroup
\newtheorem{theorem}{Theorem}[section]
\newtheorem*{theorem*}{Theorem}
\newtheorem*{"theorem"}{``Theorem''}
\newtheorem{corollary}[theorem]{Corollary}

\newtheorem{lemma}[theorem]{Lemma}
\endgroup

\theoremstyle{definition}
\begingroup
\newtheorem{definition}[theorem]{Definition}
\endgroup

\theoremstyle{remark}
\begingroup
\newtheorem{remark}[theorem]{Remark}
\newtheorem{example}[theorem]{Example}
\endgroup 

\numberwithin{equation}{section}
\newcommand{\N}{\mathbb N} 
 
\newcommand{\Z}{\mathbb Z} 
\newcommand{\R}{\mathbb R} 

\newcommand{\dist}{{\rm dist}}

\newcommand{\spt}{{\mathrm{spt}}}

\renewcommand{\H}{{\mathcal H}}
\newcommand{\E}{{\mathbb E}}
\newcommand{\W}{{\mathcal W}}
\newcommand{\M}{{\mathcal M}}

\newcommand{\Q}{{\mathcal Q}}
\renewcommand{\L}{{\mathcal L}}

\newcommand{\F}{{\mathcal F}}
\newcommand{\G}{{\mathcal G}}

\newcommand{\B}{\mathcal{B}}
\renewcommand{\P}{\mathbb{P}}
\newcommand{\Rad}{\operatorname{Rad}}

\newcommand{\Ra} {\Rightarrow}
\newcommand{\wto}{\rightharpoonup}
\newcommand{\embeds}{\xhookrightarrow{\quad}}

\renewcommand{\d}{\mathrm{d}}

\newcommand{\dx}{\,\mathrm{d}x}

\newcommand{\ds}{\,\mathrm{d}s}
\newcommand{\dt}{\,\mathrm{d}t}

\newcommand{\bx}{\mathbf{x}}

\newcommand{\A}{\mathcal{A}}

\newcommand{\eps}{\varepsilon}
\newcommand{\average}{{\mathchoice {\kern1ex\vcenter{\hrule height.4pt
width 6pt depth0pt} \kern-9.7pt} {\kern1ex\vcenter{\hrule
height.4pt width 4.3pt depth0pt} \kern-7pt} {} {} }}

\newcommand{\Risk}{\mathcal{R}}
\allowdisplaybreaks
\newcommand\showlabel{\addtocounter{equation}{1}\tag{\theequation}}

\makeatletter
\@namedef{subjclassname@2020}{2020 Mathematics Subject Classification}
\makeatother

\renewcommand{\paragraph}[1]{\textbf{#1.}}
\begin{document}

\title[Banach spaces of wide multi-layer neural networks]{On the Banach spaces associated with multi-layer ReLU networks
\\\footnotesize{Function representation, approximation theory and gradient descent dynamics}}

\author{Weinan E}
\address{Weinan E\\
Department of Mathematics and Program in Applied and Computational Mathematics\\ Princeton University\\ Princeton, NJ 08544\\ USA}
\email{weinan@math.princeton.edu}

\author{Stephan Wojtowytsch}
\address{Stephan Wojtowytsch\\
Program in Applied and Computational Mathematics\\
Princeton University\\
Princeton, NJ 08544\\
USA
}
\email{stephanw@princeton.edu}

\date{\today}


\begin{abstract}
We develop Banach spaces for ReLU neural networks of finite depth $L$ and infinite width. The spaces contain all finite fully connected $L$-layer networks and their $L^2$-limiting objects under bounds on the natural path-norm. 
Under this norm,  the unit ball in the space for $L$-layer networks has low Rademacher complexity and thus favorable generalization properties.
Functions in these spaces can be approximated by multi-layer neural networks with dimension-independent convergence rates.

The key to this work is a new way of representing functions in some form of expectations, motivated by multi-layer neural networks.
This representation allows us to define a new class of continuous models for machine learning.
We show that the gradient flow defined this way is the natural continuous analog of the gradient descent
dynamics for the associated multi-layer neural networks.
We show  that the path-norm increases at most polynomially under this continuous gradient flow dynamics.
\end{abstract}

\maketitle

\setcounter{tocdepth}{1}
\tableofcontents

\section{Introduction}


It is well-known that neural networks can approximate any continuous function on a compact set arbitrarily well in the uniform topology as the number of trainable parameters increase \cite{cybenko1989approximation, hornik1991approximation, leshno1993multilayer}.
However,  the number and magnitude of the parameters required may make this result unfeasible for practical applications. 
Indeed it has been shown to be the case when two-layer neural networks are used to approximate general Lipschitz continuous functions \cite{approximationarticle}.
It is therefore necessary to ask which functions can be approximated {\em well} by neural networks,
by which we mean that as the number of parameters goes to infinity, the convergence rate should not suffer from the curse of dimensionality.

In classical approximation theory, the role of neural networks was taken by (piecewise) polynomials or Fourier series and the natural 
function spaces were H\"older spaces, (fractional) Sobolev spaces, or generalized versions thereof \cite{lorentz1966approximation}. In the high-dimensional theories characteristic for machine learning, these spaces appear inappropriate (for example, approximation results of the kind
discussed above do not hold for these spaces) and other concepts have emerged, such as reproducing kernel Hilbert spaces for random feature models
 \cite{rahimi2008uniform}, Barron spaces for two-layer neural networks \cite{weinan2019lei, bach2017breaking,barron_new,E:2019aa, approximationarticle, E:2018ab,klusowski2016risk},  and the flow-induced space for residual neural network models \cite{weinan2019lei}.
 


In this article, we extend these ideas to networks with several hidden (infinitely wide) layers. 
The key is to find how functions in these spaces should be represented and what the right norm should be.
Our most important results are:

\begin{enumerate}
\item There exists a class of Banach spaces associated to multi-layer neural networks which has low Rademacher complexity 
(i.e.\ multi-layer functions in these spaces are easily learnable). 
\item The neural tree spaces introduced here  are the appropriate function spaces for the corresponding multi-layer neural networks in terms of direct and inverse approximation theorems.
\item The gradient flow dynamics  is well defined in a much simpler subspace of the corresponding neural tree space.
Functions in this space admit an intuitive representation in terms of compositions of expectations.
 The path norm increases at most polynomially in time under the natural gradient flow  dynamics.

\end{enumerate}

These results justify our choice of function representation and the norm.

Neural networks are parametrized by weight matrices which share indices only between adjacent layers. To understand the approximation power of neural networks, we rearrange the index structure of weights in a tree-like fashion and show that the approximation problem under path-norm bounds remains unchanged. This approach makes the problem more linear and easier to handle from the approximation perspective, but is unsuitable when describing training dynamics. To address this discrepancy, we introduce a subspace of the natural function spaces for very wide multi-layer neural networks (or neural trees) which automatically incorporates the structure of neural networks. For this subspace, we investigate 
the natural  training dynamics and demonstrate that the path-norm increases at most polynomially during training.

Although the function representation and function spaces are motivated by developing an approximation theory for multi-layer neural network models,
once we have them,  we can use them as our starting point for developing alternative machine learning models and algorithms.
In particular, we can extend  the program proposed in \cite{E:2019aa} on continuous formulations of machine learning 
to function representations developed here.
 As an example, we show that gradient descent training  for multi-layer neural networks can be recovered as the discretization of a 
 natural continuous gradient flow. 

The article is organized as follows. In the remainder of the introduction, we discuss the philosophy behind this study and the continuous approach to machine learning. In Section \ref{section generalized Barron}, we motivate the `neural tree' approach, introduce an abstract class of function spaces and study their first properties. A special instance of this class tailored to multi-layer networks is studied in greater detail in Section \ref{section multi-layer}. A class of function families with an explicit network structure is introduced in Section \ref{section indexed}. While Sections \ref{section generalized Barron} and \ref{section multi-layer} are written from the approximation perspective, Section \ref{section mean-field} is devoted to the study of gradient flow optimization of multi-layer networks and its relation to the function spaces we introduce. 
We conclude the article with a brief discussion of our results and some open questions in Section \ref{section conclusion}. Technical results from measure theory which are needed in the article are gathered in the appendix.

\subsection{Conventions and notation}
Let $K\subseteq \R^d$ be a compact set. Then we denote by $C^0(K)$ the space of continuous functions on $K$ and by $C^{0,\alpha}(K)$ the space of $\alpha$-H\"older continuous functions for $\alpha\in(0,1]$. In particular $C^{0,1}$ is the space of Lipschitz-continuous functions. The norms are denoted as
\[
\|f\|_{C^0(K)} = \sup_{x\in K} |f(x)|, \quad [f]_{C^{0,\alpha}(K)} = \sup_{x, y\in K, x\neq y} \frac{|f(x)-f(y)|}{|x-y|^\alpha}, \quad \|f\|_{C^{0,\alpha}} = \|f\|_{C^0} + [f]_{C^{0,\alpha}}.
\]
Since all norms on $\R^d$ are equivalent, the space of H\"older- or Lipschitz-continuous functions does not depend on the choice of norm on $\R^d$. The H\"older constant $[\cdot]_{C^{0,\alpha}}$ however does depend on it, and using different $\ell^p$-norms leads to a dimension-dependent factor. In this article, we consider always consider $\R^d$ equipped with the $\ell^\infty$-norm. 

Let $X$ be a Banach space. Then we denote by $B^X$ the closed unit ball in $X$. Furthermore, a review of notations, terminologies and results relating to measure theory can be found in the appendix.

Frequently and without comment, we identify $x\in \R^d$ with $(x,1)\in \R^{d+1}$. This allows us to simplify notation and treat affine maps as linear. In particular, for $x\in \R^d$ and $w\in \R^{d+1}$ we simply write $w^Tx= \sum_{i=1}^d w_i x_i + w_{d+1}$.

\section{Generalized Barron spaces}\label{section generalized Barron}

We begin by reviewing multi-layer neural networks.

\subsection{Neural networks and neural trees}\label{section network vs tree}

A {\em fully connected L-layer neural network} is a function of the type
\begin{equation}\label{eq nn intro}
f(x) = \sum_{i_L=1}^{m_L}a^L_{i_L}\sigma\left(\sum_{i_{L-1}=1}^{m_{L-1}} a^{L-1}_{i_Li_{L-1}}\sigma\left(\sum_{i_{L-2}} \dots \sigma\left(\sum_{i_1=1}^{m_1}a_{i_2i_1}^1\,\sigma\left(\sum_{i_0=1}^{d+1} a^0_{i_1i_0}\,x_{i_0}\right)\right)\right)\right)
\end{equation}
where the parameters $a_{ij}^\ell$ are referred to as the {\em weights} of the neural network, $m_\ell$ is the {\em width} of the $\ell$-th layer, and $\sigma:\R\to\R$ is a non-polynomial activation function. For the purposes of this article, we take $\sigma$ to be the {\em rectifiable linear unit} $\sigma(z) = \operatorname{ReLU}(z) = \max\{z,0\}$.

Deep neural networks are complicated functions of both their input $x$ and their weights, where the weights of one layers only share an index with neighbouring layers, leading to parameter reuse. For simplicity, consider a network with two hidden layers
\[
f(x) = \sum_{i_2=1}^{m_2} a^2_{i_2}\sigma\left(\sum_{i_1=1}^{m_1}a^1_{i_2i_1}\sigma\left(\sum_{i_0=1}^{d+1} a^0_{i_1i_0} x_{i_0}\right)\right)
\]
and note that $f$ can also be expressed as
\[
f(x) = \sum_{i_2=1}^{m_2} a^2_{i_2}\sigma\left(\sum_{i_1=1}^{m_1}a^1_{i_2i_1}\sigma\left(\sum_{i_0=1}^{d+1} b^0_{i_2i_1i_0} x_{i_0}\right)\right)
\]
with $b^0_{i_2i_1i_0} \equiv a^0_{i_1i_0}$. In this way, an index in the outermost layer gets its own set of parameters for deeper layers, eliminating parameter sharing. The function parameters are arranged in a tree-like structure rather than a network with many cross-connections. On the other hand, a function of the form 
\[
f(x) = \sum_{i_2=1}^{m_2} a^2_{i_2}\sigma\left(\sum_{i_1=1}^{m_1}a^1_{i_2i_1}\sigma\left(\sum_{i_0=1}^{d+1} a^0_{i_2i_1i_0} x_{i_0}\right)\right)
\]
can equivalently be expressed as 
\[
f(x) = \sum_{i_2=1}^{m_2} a^2_{i_2}\sigma\left(\sum_{j_1=1}^{m_1m_2}b^1_{i_2j_1}\sigma\left(\sum_{i_0=1}^{d+1} b^0_{j_1i_0} x_{i_0}\right)\right)
\]
with 
\[
b^1_{i_2j_1} = \begin{cases} a^1_{i_2,\, j_1- (i_2-1)m_1} &\text{if }(i_2-1)m_1 < j_1 \leq i_2 m_1\\
	0 &\text{else}\end{cases},
	\qquad
b^0_{j_1i_0} = a_{\lfloor j_1/ m_1\rfloor+1, j_1 - \lfloor j_1/ m_1\rfloor, i_0}.
\]
The cost of rearranging a three-dimensional index set into a two-dimensional one is listing a number of zero-elements explicitly in the preceding layer instead of implicitly. Conversely, if we rearrange a two-dimensional index set into a three-dimensional one, we need to repeat the same weight multiple times. For deeper trees, the index sets become even higher-dimensional, and the re-arrangement introduces even more trivial branches or redundancies. Nevertheless, we note that the space of finite neural networks of depth $L$
\begin{align*}
\F_\infty &:= \Bigg\{\sum_{i_L=1}^{\infty}a^L_{i_L}\sigma\left(\sum_{i_{L-1}=1}^{\infty} a^{L-1}_{i_Li_{L-1}}\sigma\left(\sum_{i_{L-2}} \dots \sigma\left(\sum_{i_1=1}^{\infty}a_{i_2i_1}^1\,\sigma\left(\sum_{i_0=1}^{d+1} a^0_{i_1i_0}\,x_{i_0}\right)\right)\right)\right)\\
	&\hspace{6.5cm}\Bigg| \: a^l_{ij} = 0 \quad\text{ for all but finitely many }i,j,l\Bigg\}
\end{align*}
and the space of finite neural trees of depth $L$
\begin{align*}
\widetilde \F_\infty &:= \Bigg\{\sum_{i_L=1}^{\infty}a^L_{i_L}\sigma\left(\sum_{i_{L-1}=1}^{\infty} a^{L-1}_{i_Li_{L-1}}\sigma\left(\sum_{i_{L-2}} \dots \sigma\left(\sum_{i_1=1}^{\infty}a_{i_L\dots i_2i_1}^1\,\sigma\left(\sum_{i_0=1}^{d+1} a^0_{i_L\dots i_1i_0}\,x_{i_0}\right)\right)\right)\right)\\
	&\hspace{6cm}\:\Bigg|\: a^l_{i_L\dots i_k} = 0 \quad\text{ for all but finitely many }l, i_1,\dots, i_L\Bigg\}
\end{align*}
are identical.

\begin{remark}
We note that this perspective is only admissible concerning approximation theory. For gradient flow-based training algorithms, it makes a huge difference
\begin{itemize}
\item whether parameters are reused or not,
\item which set of weights that induces a certain function is chosen, and
\item how the magnitude of the weights is distributed across the layers (using the invariance $\sigma(z) = \lambda^{-1} \sigma ( \lambda z)$ for $\lambda>0$).
\end{itemize}
A perspective more adapted to the training of neural networks is presented in Section \ref{section mean-field}.
\end{remark}

For given weights $a_{ij}^l$ or $a^l_{i_L\dots i_l}$, we consider the {\em path-norm proxy}, which is defined as
\[
\|f\|_{pnp} = \sum_{i_L} \dots\sum_{i_0} \big| a^L_{i_L}\dots a^0_{i_1i_0}\big|\qquad \text{or}\qquad \|f\|_{pnp} = \sum_{i_L} \dots\sum_{i_0} \big| a^L_{i_L}\dots a^0_{i_L\dots i_0}\big|
\]
respectively. Knowing the weights, the sum is easy to compute and it naturally controls the Lipschitz norm of the function $f$. 

When we train a function $f$ to approximate values $y_i = f^*(x_i)$ at data points $x_i$, the path-norm proxy controls the {\em generalization error},
as we will show below. If the path-norm proxy of $f$ is very large, the function values $f(x_i)$ heavily depend on cancellations between the partial sums with positive and negative weights in the outermost layer. In the extreme case, these partial sums may be several orders of magnitude larger than $f(x_i)$. In that situation, the function values $f^*(x)$ and $f(x)$ may be entirely different for unseen data points $x$, even if they are close on the training sample $\{x_i\}_{i=1}^N$. On the other hand, we will show below that functions with low path-norm proxy generalize well. Thus controlling the path-norm proxy effectively means  controlling the generalization error,
either directly or indirectly. We will make this more precise below.

While the path-norm proxy is easy to compute from the weights of a network, it is a quantity related to the parameterization of a function, not the function itself. The map from the weights $a^l_{ij}$ to the {\em realization} $f$ of the network  as in \eqref{eq nn intro} is highly non-injective. The {\em path-norm} of a function $f$ is the infimum of the path-norm proxies over all sets of weights of an $L$-layer neural network which have the realization $f$. 

\subsection{Definition of Generalized Barron Spaces}

Let $\sigma$ be the rectified linear unit, i.e.\ $\sigma(z) = \max\{z,0\}$. ReLU is a popular activation function for neural networks and has two useful properties for us: It is positively one-homogeneous and Lipschitz continuous with Lipschitz constant $1$.

Let $K\subseteq \R^d$ be a compact set and $X$ be a Banach space such that

\begin{enumerate}
\item $X$ embeds continuously into the space $C^{0,1}(K)$ of Lipschitz-functions on $K$ and
\item the closed unit ball $B^X$ in $X$ is closed in the topology of $C^{0}(K)$. 
\end{enumerate}

Recall the following corollary to the Arzel\`a-Ascoli theorem.

\begin{lemma}\cite[Satz 2.42]{dobrowolski2010angewandte}\label{lemma compact embedding holder}
Let $u_n:K\to \R$ be a sequence of functions such that $\|u_n\|_{C^{0,1}(K)}\leq 1$. Then there exists $u\in C^{0,1}(K)$ and a subsequence $u_{n_k}$ such that $u_{n_k}\to u$ strongly in $C^{0,\alpha}(K)$ for all $\alpha<1$ and 
\[
\|u\|_{C^{0,1}(K)} \leq \liminf_{k\to\infty} \|u_{n_k}\|_{C^{0,1}(K)}\leq 1.
\]
\end{lemma}

Thus $B^X$ is pre-compact in the separable Banach space $C^0(K)$. Since $B^X$ is $C^0$-closed, it is compact, so in particular a Polish space. A brief review of measure theory in Polish spaces and related topics used throughout the article is given in Appendix \ref{appendix measure theory}.

Let $\mu$ be a finite signed measure on the Borel $\sigma$-algebra of $B^X$ (with respect to the $C^0$-norm). Then $\mu$ is a signed Radon measure. The vector-valued function
\[
B^X\to C^0(K),\qquad g\mapsto \sigma(g)
\]
is continuous and thus $\mu$-integrable in the sense of Bochner integrals. We define
\begin{align*}
f_{\mu} &= \int_{B^X} \sigma\big(g(\cdot)\big)\,\mu(\d g)\\\showlabel
\|f\|_{X, K} &= \inf\left\{ \|\mu\|_{\M(B^X)}\::\: \mu \in \M(B^X)\text{ s.t. }f = f_{\mu} \text{ on }K\right\}\\
\B_{X, K} &= \big\{f\in C^{0}(K) : \|f\|_{X, K} < \infty\big\}.
\end{align*}
Here $\M(B^X)$ denotes the space of (signed) Radon measures on $B^X$. The first integral can equivalently be considered as a Lebesgue integral pointwise for every $x\in K$ or as a Bochner integral. We will show below that $\B_{X,K}$ is a normed vector space of (Lipschitz-)continuous functions on $K$. We call $\B_{X,K}$ the generalized Barron space modelled on $X$. 

\begin{remark}
The construction of the function space $\B_{X,K}$ above resembles the approach to {\em Barron spaces} for two-layer networks \cite{bach2017breaking, barron_new, weinan2019lei, E:2018ab}. Note that Barron spaces are distinct from the class of functions considered by Barron in \cite{barron1993universal}, which is sometimes referred to as {\em Barron class}. While Barron spaces are specifically designed for applications concerning neural networks, the Barron class is defined in terms of spectral properties and a subset of Barron space for almost every activation function of practical importance.
\end{remark}

\begin{example}\label{example barron barron}
If $X$ is the space of affine functions from $\R^d$ to $\R$ (which is isomorphic to $\R^{d+1}$), the $\B_{X,K}$ is the usual Barron space for two-layer neural networks as described in \cite{E:2018ab, weinan2019lei, barron_new}. 
\end{example}

Due to Lemma \ref{lemma compact embedding holder}, we may choose $X= C^{0,1}(K)$.

\begin{example}\label{example barron lipschitz}
If $X= C^{0,1}(K)$, then $\B_{X,K} = C^{0,1}(K)$ and the norms are equivalent to within a factor of two. For $f\in C^{0,1}(K)$, we represent 
\begin{align*}
f &= \|f\|_{C^{0,1}(K)} \,\sigma\left(\frac{f}{\|f\|_{C^{0,1}(K)}}\right)  - \|f\|_{C^{0,1}(K)} \,\sigma\left(\frac{f}{\|f\|_{C^{0,1}(K)}}\right)\\
	&= \int_{B^X} \sigma(g)\,\left(\|f\|_{C^{0,1}}\cdot \delta_{\frac{f}{\|f\|_{C^{0,1}}}} - \|f\|_{C^{0,1}}\cdot\delta_{ -\frac{f}{\|f\|_{C^{0,1}}}}\right)(\d g).
\end{align*}
\end{example}

These examples are on opposite sides of the spectrum with $X$ being either the least complex non-trivial space or the largest admissible space. Spaces of deep neural networks lie somewhere between those extremes.

\begin{remark}
For  the classical Barron space, we usually consider measures supported on the unit sphere in the finite-dimensional space $X$. If $X$ is infinite-dimensional, typically only the unit ball in $X$ is closed (and thus compact) in $C^0$, but not the unit sphere. For mathematical convenience, we choose the compact setting. 
\end{remark}

\subsection{Properties}

Let us establish some first properties of generalized Barron spaces.

\begin{theorem}\label{theorem first properties}
The following are true.
\begin{enumerate}
\item $\B_{X,K}$ is a Banach-space.
\item $X\embeds \B_{X,K}$ and $\|f\|_{\B_{X,K}} \leq 2\,\|f\|_X$.
\item $\B_{X,K}\embeds C^{0,1}(K)$ and the closed unit ball of $\B_{X,K}$ is a closed subset of $C^{0}(K)$.
\end{enumerate}
\end{theorem}

\begin{proof}
Since $X\embeds C^{0,1}(K)$, we know that there exist $C_1, C_2>0$ such that
\[
\|g\|_{C^0(K)} \leq C_1\,\|g\|_{X}, \qquad [g]_{C^{0,1}(K)} \leq C_2\,\|g\|_X\qquad \forall\ g\in X.
\]

\paragraph{Banach space} By construction, $\B_{X,K}$ is isometric to the quotient space $\M(B^X)/N_K$ where 
\[
N_K = \left\{ \mu \in \M(B^X) \:\bigg|\: \int_{B^X} \sigma\big(g(x)\big)\,\mu(\d g) = 0\:\forall\ x\in K\right\}.
\]
In particular, $\B_{X,K}$ is a normed vector space with the norm $\|\cdot\|_{X,K}$. The map
\[
\M(B^X) \to C^0(K), \qquad \mu\mapsto f_\mu =  \int_{B^X} \sigma(g)\,\mu(\d g)
\]
is continuous as
\[
\left\|\int_{B^X} \sigma(g)\,\mu(\d g)\right\|_{C^0(K)} \leq \int_{B^X} \|g\|_{C^0(K)}\,|\mu|(\d g) \leq C_1\,\|\mu\|_{\M(B^X)}
\]  
by the properties of Bochner spaces. Thus $N_K$ is the kernel of a continuous linear map, i.e.\ a closed closed subspace. We conclude that $\B_{X,K}$ is a Banach space \cite[Proposition 11.8]{MR2759829}.

\paragraph{$X$ embeds into $\B_{X,K}$} For $g\in X$ with $\|g\|_X = 1$ consider $\mu = \delta_g - \delta_{-g}$ and observe that
\[
f_\mu = \sigma(g) - \sigma(-g) = g, \qquad \|\mu\|_{\M(B^X)} = 2.
\]
The general case follows by homogeneity. 

\paragraph{$\B_{X,K}$ embeds into $C^{0,1}$} We have already shown that $\|f_\mu\|_{C^0(K)} \leq C_1\,\|\mu\|_{\M(B^X)}$. By taking the infimum over $\mu$, we find that $\|f\|_{C^0(K)} \leq R\,\|f\|_{\B_{X,K}}$. Furthermore, for any $x\neq y\in K$ we have
\begin{align*}
|f_\mu(x) - f_\mu(x')| &\leq \int_{B^X} \big|\sigma\big(g(x)\big) - \sigma\big(g(x')\big)\big|\,|\mu| (\d g)\\
	&\leq \int_{ B^X} \big|g(x) - g(x')\big|\,|\mu| (\d g)\\
	&\leq \int_{B^X} [g]_{C^{0,1}} |x-x'|\,|\mu| (\d g)\\
	&\leq C_2\,\|\mu\|_{\M(B^X)}\,|x-x'|
\end{align*}
We can now take the infimum over $\mu$.

Now assume that $(f_n)_{n\in\N}$ is a sequence such that $\|f_n\|_{X,K}\leq 1$ for all $n\in\N$. Choose a sequence of measures $\mu_n$ such that $f_n = f_{\mu_n}$ and $\|\mu_n\|\leq 1 + \frac1n$ for all $n\in \N$. By the compactness theorem for Radon measures (see Theorem \ref{compactness theorem radon measures} in the appendix), there exists a subsequence $\mu_{n_k}$ and a Radon measure $\mu$ on $B^X$ such that $\mu_{n_k}\wto \mu$ as Radon measures and $\|\mu\|\leq 1$. 

By definition, the weak convergence of Radon measures implies that
\[
\int_{B^X} F(g)\,\mu_{n_k}(\d g) \to \int_{B^X} F(g)\,\mu(\d g) \qquad \forall\ F\in C(B^X).
\]
Using $F(g) = \sigma(g(x))$, we find that $f_{\mu_{n_k}}\to f_\mu$ pointwise. In particular, if $f_{\mu_n}$ converges to a limit $\tilde f$ uniformly, then $\tilde f = f \in B^{\B_{X,K}}$, i.e.\ the unit ball of $\B_{X,K}$ is closed in the $C^0$-topology.
\end{proof}

The last property establishes that $\B_{X,K}$ satisfies the same properties which we imposed on $X$, i.e.\ we can repeat the construction and consider $\B_{\B_{X,K},K}$.

\begin{remark}
We have shown in \cite{barron_new} that if $K$ is an infinite set, Barron space is generally neither separable nor reflexive. In particular, $\B_{X,K}$ is not expected to have either of these properties in the more general case.
\end{remark}

\subsection{Rademacher complexities}\label{section Rademacher general}

We show that generalized Barron spaces have a favorable property from the perspective of statistical learning theory. 

%

A convenient (and sometimes realistic) assumption is that all data samples accessible to a statistical learner are drawn from a distribution $\P$ independently. The pointwise Monte-Carlo error estimate follows from the law of large numbers which shows that for a fixed function $f$ and data distribution $\P$, we have
\[
\left|\E_{(X_1,\dots,X_N)\sim \pi^N} \left[\sum_{i=1}^N f(X_i) - \int f(x)\,\P(\d x)\right] \right|\leq \frac{C_f}{\sqrt{N}}
\]
Typically, the uniform error over a function class is much larger than the pointwise error. For example for the class of one-Lipschitz functions
\[
\left|\E_{(X_1,\dots,X_N)\sim \pi^N} \sup_{[f]_{C^{0,1}}\leq 1} \left[\sum_{i=1}^N f(X_i) - \int f(x)\,\P(\d x)\right] \right| = \E_{(X_1,\dots,X_N)\sim \pi^N} \left[W_1\left(\P, \:\frac1N \sum_{i=1}^N\delta_{X_i}\right)\right]
\]
is the expected $1$-Wasserstein distance between $\P$ and the empirical measure of $N$ independent sample points drawn from it. If $\P$ is the uniform distribution on $[0,1]^d$, this decays like $N^{-1/d}$ and thus much slower than $N^{-1/2}$ \cite{fournier2015rate, approximationarticle}.

For Barron-type spaces, the Monte-Carlo error rate may be attained {\em uniformly} on the unit ball of $\B_{X,K}$. This is established using the Rademacher complexity of a function class. Rademacher complexities essentially decouple the sign and magnitude of oscillations around the mean by introducing additional randomness in a problem. 
For general information on Rademacher complexities, see \cite[Chapter 26]{shalev2014understanding}.
 
\begin{definition}
Let $S = \{x_1,\dots, x_N\}$ be a set of points in $K$. The {\em Rademacher complexity} of $\H \subseteq C^{0,1}(K)$ on $S$ is defined as
\begin{equation}
\Rad(\H;S) = \E_{\xi} \left[\sup_{h\in\H}\frac1N\sum_{i=1}^N\xi_i\,h(x_i)\right]
\end{equation}
where the $\xi_i$ are iid random variables which take the values $1$ and $-1$ with probability $1/2$ each. 
\end{definition}

The $\xi_i$ are either referred to as symmetric Bernoulli or Rademacher variables, depending on the author.

\begin{theorem}\label{theorem rademacher general}
Denote by $\F$ the unit ball of $\B_{X,K}$. Let $S$ be any sample set in $\R^d$. Then
\[
\Rad(\F; S) \leq 2\,\Rad(B^X, S).
\]
\end{theorem}

\begin{proof}
Define the function classes $\H_1 = \{\sigma(g): g\in B^X\}$, $\H_2 = \{-\sigma(g) : g\in B^X\}$ and $\H = \H_1\cup\H_2$. All three are compact in $C^0$.

We decompose $\mu = \mu^+-\mu^-$ in its mutually singular positive and negative parts and write $f=f_\mu$ in $\B_{X,K}$ as 
\begin{align*}
f_\mu(x) &= \int_{B^X} \sigma(g(x))\,\mu^+(\d g) + \int_{B^X}\,-\sigma(g(x))\,\mu^-(\d g)\\
	&= \int_{\H_1} h(x)\,(\rho^+_\sharp \mu^+)(\d h) + \int_{\H_2} h(x)\,(\rho^-_\sharp \mu^-)(\d h)\\
	&= \int_{\H} h(x)\,\hat\mu(\d h)
\end{align*}
where $\rho^\pm:B^X\to \H$ is given by $g\mapsto \pm \sigma(g)$ and $\hat \mu = \rho^+_\sharp\mu^+ + \rho^-_\sharp \mu^-$. In particular, we note that $\hat\mu$ is a non-negative measure and $\|\hat\mu\| = \|\mu\|$. We conclude that the closed unit ball in $\B_{X,K}$ is the closed convex hull of $\H$.

Since $\sigma$ is $1$-Lipschitz, the contraction Lemma \cite[Lemma 26.9]{shalev2014understanding} implies that $\Rad(\H_1; S) \leq \Rad(B^X; S)$.
Due to \cite[Lemma 26.7]{shalev2014understanding}, we find that 
\begin{align*}
\Rad(B^{\B_{X,K}}; S) &= \Rad(\H; S)\\
	&= \Rad(\H_1\cup (-\H_1); S)\\
	&\leq \Rad(\H_1; S) + \Rad(-\H_1; S)\\
	&= 2\, \Rad(\H_1; S)\\
	&= 2\, \Rad(B^X; S)
\end{align*}
since for any $\xi$, the supremum is non-negative.
\end{proof}

For a priori estimates, it suffices to bound the expected Rademacher complexity. However, the use of randomness in the problem is complicated, and most known bounds work on any suitably bounded sample set.

\begin{example}\label{example rademacher linear}
If $\H_{lin}$ is the class of linear functions on $\R^d$ with $\ell^1$-norm smaller or equal to $1$ and $S$ is any sample set of $N$ elements in $[-1,1]^d$, then
\[
\Rad(\H_{lin}; S) \leq \sqrt{\frac{2\,\log(2d)}N},
\]
see \cite[Lemma 26.11]{shalev2014understanding}. If $\H_{aff}$ is the unit ball in the class of affine functions $x\mapsto w^Tx+b$ with the norm $|w|_{\ell^1}+|b|$, we can simply extend $x$ to $(x,1)$ and see that
\[
\Rad(\H_{aff}; S) \leq \sqrt{\frac{2\,\log(2d+2)}N}.
\]
\end{example}

We show that Monte-Carlo rate decay is the best possible result for Rademacher complexities under very weak conditions.

\begin{example}\label{example rademacher complexity lower bound}
Let $\F$ be a function class which contains the constant functions $f\equiv -1$ and $f\equiv 1$ for $\alpha,\beta\in\R$. Then there exists $c>0$ such that
\[
\Rad(\F;S)\geq c\frac{|\alpha-\beta|}{\sqrt N}
\]
for any sample set $S$ with $N$ elements. Up to scaling and a constant shift (which does not affect the complexity), we may assume that $\beta =1, \alpha = -1$. Then
\begin{align*}
\Rad(\F;S) &\geq \E_\xi \frac1m \sup_{f\equiv \pm 1} \sum_{i=1}^m \xi_if(x_i)\\
	&= \E_\xi \frac1m \left|\sum_{i=1}^m \xi_i\right|\\
	&\sim \frac1{\sqrt{2\pi m}}
\end{align*}
by the central limit theorem.
\end{example}

\section{Banach spaces for multi-layer neural networks}\label{section multi-layer}

\subsection{neural tree spaces}

In this section, we discuss feed-forward neural networks of infinite width and finite depth $L$. Let $K\subseteq\R^d$ be a fixed compact set. Consider the following sequence of spaces.
\begin{enumerate}
\item $\W^0(K) = (\R^d)^*\oplus \R \:\widetilde = \:\R^{d+1}$ is the space of affine functions from $\R^d$ to $\R$ (restricted to $K$).
\item For $L\geq 1$, we set $\W^L(K)= \B_{\W^{L-1}(K),K}$. 
\end{enumerate}
Since we consider $\R^d$ to be equipped with the $\ell^\infty$-norm, we take $\W^0$ to be equipped with its dual, the $\ell^1$-norm. Up to a dimension-dependent normalization constant, this does not affect the analysis.

Thus $\W^L$ is the function space for $L+1$-layer networks (i.e.\ networks with $L$ hidden layers/nonlinearities). Here we use inductively that $\W^L$ embeds into $C^{0,1}(K)$ continuously and that the unit ball of $\W^L$ is $C^0$-closed because the same properties held true for $\W^{L-1}$. Due to the tree-like recursive construction, we refer to $\W^L$ as neural tree space (with $L$ layers).  

Here and in the following, we often assume that $K$ is a fixed set and will suppress it in the notation $\W^L = \W^L(K)$. 

\begin{remark}
For a network with one hidden layer, by construction the coefficients in the inner layer are $\ell^\infty$-bounded, while the outer layer is bounded in $\ell^1$ (namely as a measure). Due to the homogeneity of the ReLU activation function, the bounds can be easily achieved and the function space is not reduced compared to just requiring the path-norm proxy to be finite.

For other activation functions, an $\ell^\infty$-bound on the coefficients in the inner layer may restrict the space of functions which can be approximated. In particular, if $\sigma$ is $C^k$-smooth, then $x\mapsto a\,\sigma(w^Tx)$ is $C^k$-smooth uniformly in $w\in B_R(0)\subseteq \R^{d+1}$. As a consequence, the space of $\sigma$-activated two-layer networks whose inner layer coefficients are $\ell^\infty$-bounded embeds continuously into $C^k$. At least if $k> d/2$, it follows from \cite{barron1993universal} that this space is smaller than the space of functions which can be approximated by $\sigma$-activated two-layer networks with uniformly bounded path-norm (see also \cite[Theorem 3.1]{barron_new}).

It is likely that neural tree spaces with more general activation require parametrization by Radon measures on entire Banach spaces of functions. For networks with a single hidden layer, some results in this direction were presented in the appendix of \cite{approximationarticle}. While Radon measures on $\R^{d+2}$ are less convenient than those on $S^{d+1}$, many results can be carried over since $\R^{d+2}$ is locally compact. 

The situation is very different for networks with two hidden layers. The space $X = \W^1$ on which $\W^2 = \B_X$ is modelled is infinite-dimensional, dense in $C^0$, and not locally compact in the $C^0$-topology. The restriction to the compact set $B^X$ simplifies the analysis considerably.
\end{remark}

\subsection{Embedding of finite networks}\label{section finite networks}
The space $\W^L$ contains all finite networks with $L\geq 1$ hidden layers.

\begin{theorem}\label{embedding theorem}
Let 
\begin{equation}\label{eq L layer network}
f(x) = \sum_{i_L=1}^{m_L}a^L_{i_L}\sigma\left(\sum_{i_{L-1}=1}^{m_{L-1}} a^{L-1}_{i_Li_{L-1}}\sigma\left(\sum_{i_{L-2}} \dots \sigma\left(\sum_{i_1=1}^{m_1}a_{i_2i_1}^1\,\sigma\left(\sum_{i_0=1}^{d+1} a^0_{i_1i_0}\,x_{i_0}\right)\right)\right)\right)
\end{equation}
Then $f\in \W^L$ and 
\begin{equation}
\|f\|_{\W^L}\leq \sum_{i_L=1}^{m_L}\dots\sum_{i_1=1}^{m_1} \sum_{i_0=1}^{d+1} \big|a^L_{i_L}\,a_{i_Li_{L-1}}^{L-1}\,\dots a^0_{i_1i_0}\big|
\end{equation}
\end{theorem}

\begin{proof}
The statement is obvious for $L=1$ as
\[
f(x) = \sum_{i=1}^{m_1} a_{i_1} \sigma\left(\sum_{i_0=1}^{d+1} a_{i_1,i_0}x_{i_0}\right) = \int_{S^{d}} \sigma(w^Tx)\,\left(\sum_{i=1}^{m} a_{i}|w_i|\cdot \delta_{w_i/|w_i|}\right) (\d w).
\]
is a classical Barron function, where we simplified notation by setting $w_i = (a_{i1},\dots, a_{i(d+1)})\in \R^{d+1}$. We proceed by induction.

Let $f$ be like in \eqref{eq L layer network}. By the induction hypothesis, for any fixed $1\leq i_L\leq m_L$, the function
\[
g_{I_L}(x):= \sum_{i_{L-1}=1}^{m_{L-1}} a^{L-1}_{i_Li_{L-1}}\sigma\left(\sum_{i_{L-2}} \dots \sigma\left(\sum_{i_1=1}^{m_1}a_{i_2i_1}^1\,\sigma\left(\sum_{i_0=1}^{d+1} a^0_{i_1i_0}\,x_{i_0}\right)\right)\right)
\]
lies in $\W^{L-1}$ with the appropriate norm bound. We note that
\[
f(x) = \sum_{i_L=1}^{m_L} a_{i_L} \sigma\big( g_{i_L}(x)\big) = \sum_{i_L=1}^{m_L} \bar a_{i_L} \sigma\big( \bar g_{i_L}(x)\big)
\]
where 
\begin{align*}
\bar g_{i_L} &= \frac{g_{i_L}}{\sum_{i_{L-1}=1}^{m_L}\dots\sum_{i_1=1}^{m_1} \sum_{i_0=1}^{d+1} \big|a^L_{i_{L-1}}\,a_{i_{L-1}i_{L-2}}^{L-1}\,\dots a^0_{i_1i_0}\big|}\\
\bar a_{i_L} &= a_{i_L}\,\sum_{i_{L-1}=1}^{m_L}\dots\sum_{i_1=1}^{m_1} \sum_{i_0=1}^{d+1} \big|a^L_{i_{L-1}}\,a_{i_{L-1}i_{L-2}}^{L-1}\,\dots a^0_{i_1i_0}\big|.
\end{align*}
It follows that $f\in \W^L$ with appropriate norm bounds.
\end{proof}

\subsection{Inverse Approximation}\label{section inverse approximation}
We show that $\W^L$ does not only contain all finite ReLU networks with $L$ hidden layers, but also their limiting objects.

\begin{theorem}[Compactness Theorem]\label{compactness theorem}
Let $f_n$ be a sequence of functions in $\W^L$ such that $C^L:= \liminf_{n\to\infty} \|f_n\|_{\W^L}<\infty$. Then there exists $f\in \W^L$ and a subsequence $f_{n_k}$ such that $\|f\|_{\W^L}\leq C^L$ and $f_{n_k}\to f$ strongly in $C^{0,\alpha}(K)$ for all $\alpha<1$.
\end{theorem}

\begin{proof}
The result is trivial for $L=0$ since $\W^0$ is a finite-dimensional linear space. Using the third property from Theorem \ref{theorem first properties} inductively, we find that $\W^L$ embeds continuously into $C^{0,1}$, thus compactly into $C^{0,\alpha}$ for all $\alpha<1$. This establishes the existence of a convergent subsequence. Since $B^{\W^L}$ is $C^0$-closed, it follows that the limit lies in $\W^L$.
\end{proof}

\begin{corollary}[Inverse Approximation Theorem]\label{inverse approximation corollary}
Let
\[
f_n(x) = \sum_{i_L=1}^{m_{n,L}}a^{n,L}_{i_L}\sigma\left(\sum_{i_{L-1}=1}^{m_{n,L-1}} a^{n,L-1}_{i_Li_{L-1}}\sigma\left(\sum_{i_{L-2}} \dots \sigma\left(\sum_{i_1=1}^{m_{n,1}}a_{i_2i_1}^{n,1}\,\sigma\left(\sum_{i_0=1}^{d+1} a^{n,0}_{i_1i_0}\,x_{i_0}\right)\right)\right)\right)
\]
be finite $L$-layer network functions such that
\[
\sup_{n\in \N} \sum_{i_L=1}^{m_{n,L}}\dots\sum_{i_1=1}^{m_{n,1}} \sum_{i_0=1}^{d+1} \big|a^{n,L}_{i_L}\,a_{i_Li_{L-1}}^{n,L-1}\,\dots a^{n,0}_{i_1i_0}\big|< \infty.
\]
If $\P$ is a compactly supported probability measure and $f\in L^1(\P)$ such that $f_n\to f$ in $L^1(\P)$, then $f\in \W^L(\spt\, \P)$ and
\begin{equation}
\|f\|_{\W^L(\spt\,\P)} \leq \liminf_{n\to\infty}\sum_{i_L=1}^{m_{n,L}}\dots\sum_{i_1=1}^{m_{n,1}} \sum_{i_0=1}^{d+1} \big|a^{n,L}_{i_L}\,a_{i_Li_{L-1}}^{n,L-1}\,\dots a^{n,0}_{i_1i_0}\big|.
\end{equation}
\end{corollary}

\begin{proof}
Follows from Theorems \ref{compactness theorem} and \ref{embedding theorem}.
\end{proof}

In particular, we make no assumption whether the width of any layer goes to infinity, or at what rate. The path-norm does not control the number of (non-zero) weights of a network.

\subsection{Direct Approximation}

In Sections \ref{section finite networks} and \ref{section inverse approximation}, we showed that $\W^L$ is large enough to contain all finite ReLU networks with $L$ hidden layers and their limiting objects, even in weak topologies. In this section, we prove conversely that $\W^L$ is small enough such that every function can be approximated by finite networks with $L$ hidden layers (with rate  independent of the dimensionality), i.e.\ $\W^L$ is the smallest suitable space for these objects. 

In fact, we prove a stronger result with an approximation rate in a reasonably weak topology. The rate however depends on the number of layers. Recall the following result on convex sets in Hilbert spaces.

\begin{lemma}\cite[Lemma 1]{barron1993universal}\label{barron hilbert lemma}
Let $\G$ be a set in a Hilbert space $H$ such that $\|g\|_H \leq R$ for all $g\in\G$. If $f$ is in the closed convex hull of $\G$, then for every $m\in\N$ and $\eps>0$, there exist $m$ elements $g_1,\dots, g_m\in \G$ such that 
\begin{equation}
\left\| f - \frac1m \sum_{i=1}^m g_i \right\|_H \leq \frac{R+\eps}{\sqrt{m}}.
\end{equation}
\end{lemma}

The result is attributed to Maurey in \cite{barron1993universal} and proved using the law of large numbers. 

\begin{theorem}\label{direct approximation theorem}
Let $\P$ be a probability measure with compact support $\spt(\P)\subseteq B_R(0)$. Then for any $L\geq 1$, $f\in \W^L$ and $m\in\N$, there exists a finite $L$-layer ReLU network
\begin{equation}
f_m(x) = \sum_{i_L=1}^{m}a^L_{i_L}\sigma\left(\sum_{i_{L-1}=1}^{m^2} a^{L-1}_{i_Li_{L-1}}\sigma\left(\sum_{i_{L-2}=1}^{m^3} \dots \sigma\left(\sum_{i_1=1}^{m^{L}}a_{i_2i_1}^1\,\sigma\left(\sum_{i_0=1}^{d+1} a^0_{i_1i_0}\,x_{i_0}\right)\right)\right)\right)
\end{equation}
such that
\begin{enumerate}
\item 
\begin{equation}
\|f_m -f\|_{L^2(\P)} \leq \frac{L\,(2+R)\,\|f\|_{\W^L}}{\sqrt{m}}
\end{equation}
\item the norm bound
 \begin{equation}
\sum_{i_L=1}^{m}\dots\sum_{i_1=1}^{m^L} \sum_{i_0=1}^{d+1} \big|a^L_{i_L}\,a_{i_Li_{L-1}}^{L-1}\,\dots a^0_{i_1i_0}\big| \leq \|f\|_{\W^L}
\end{equation}
holds.
\end{enumerate}
\end{theorem}

\begin{remark}
Note that the width of deep layers increases rapidly. This is due to the fact that we construct an approximating network inductively. The procedure leads to a tree-like structure where parameters are not shared, but every neuron in the $\ell$-th layer has its own set of parameters in the $\ell+1$-th layer and $a_{i_\ell i_{\ell-1}}=0$ for all other parameter pairings. This is equivalent to standard architectures from the perspective of approximation theory under path-norm bounds, since the path norm does not control the number of neurons.

The total number of parameters in the network of the direct approximation theorem is 
\begin{align*}
M &= m + m\cdot m^2 +\dots + m^{L-1}\cdot m^L + m^L(d+1)\\
	&= \sum_{\ell=0}^{L-1}m^{2\ell+1} + m^L(d+1)\\
	&= m \frac{1-m^{2L}}{1-m^2} + m^L(d+1)\\
	&\sim m^{2L-1}
\end{align*}
by the geometric sum. Thus the decay rate in the direct approximation theorem is of the order $M^{-\frac{1}{2(2L-1)}}$. This recovers the Monte-Carlo rate $M^{-1/2}$ in the case $L=1$ \cite[Theorem 4]{weinan2019lei}, but quickly degenerates as $L$ increases. Part of the problem is that the rapidly branching structure combined with neural network indexing induces explicitly listed zeros in the set of weights as explained in Section \ref{section network vs tree}. A neural tree expressing the same function would require only $\sim (d+ L)m^L$ weights.

Note, however, that the approximation rate is independent of dimension $d$. In this sense, we are not facing a curse of dimensionality, but a curse of depth.

It is unclear whether this rate can be improved in the general setting. Functions in Barron space are described as the expectation of a suitable quantity, while multi-layer functions are described as iterated conditional expectations and non-linearities. In this setting, it is not obvious whether the Monte-Carlo rate should be expected.
\end{remark}

\begin{proof}[Proof of Theorem \ref{direct approximation theorem}]
Without loss of generality $\|f\|_{\W^L} = 1$. Since $\W^L\embeds C^{0,1}$ with constant $1$, we find that $\|f\|_{L^2(\P)} \leq (1+R)\,\|f\|_{\W^L}$ for all $f\in \W^L$. 

Recall from the proof of Theorem \ref{theorem rademacher general} that the unit ball of $\W^L$ is the closed convex hull of the class $\H = \{\pm \sigma(g) : \|g\|_{\W^{L-1}} \leq 1\}$.  Thus by Lemma \ref{barron hilbert lemma} there exist $g_1,\dots g_M\in \W^{L-1}$ and $\eps_i\in \{-1,1\}$ such that
\[
\left\|f - \frac1m \sum_{i=1}^m \eps_i\,\sigma(g_i(x))\right\|_{L^2(\P)}<\frac{2+R}{\sqrt{m}} .
\]
If $L=1$, $g_i$ is an affine linear map and $f_m(x) = \sum_{i=1}^m \frac{\eps_i}m\,\sigma(g_i(x))$ is a finite neural network. Thus the Theorem is established for $L=1$. 

We proceed by induction. Assume that the theorem has been proved for $L-1\geq 1$. Then we note that $\|g_i\|_{\W^{L-1}}\leq 1$, so for $1\leq i\leq m$ we can find a finite $L-1$-layer network $\tilde g_i$ such that
\begin{align*}
\left\|f - \frac1m \sum_{i=1}^m \eps_i\,\sigma(\tilde g_i(x))\right\|_{L^2(\P)} &\leq \left\|f - \frac1m \sum_{i=1}^M \eps_i\,\sigma(g_i(x))\right\|_{L^2(\P)} + \frac1m \sum_{i=1}^m\|g_i-\tilde g_i\|_{L^2(\P)}\\
	&\leq \frac{2+R}{\sqrt{m}} + \frac{m}{m}\frac{(L-1)(2+R)}{\sqrt m}.
\end{align*} 
We merge the $m$ trees associated with $\tilde g_i$ into a single tree, increasing the width of each layer by a factor of $m$, and add an outer layer of width $m$ with coefficients $a_{i_L} = \frac{\eps_{i_L}}m$.
\end{proof}

\begin{remark}
Let $p\in [2,\infty)$. Then by interpolation
\[
\|f-f_m\|_{L^p(\P)} \leq \|f-f_m\|_{L^2(\P)}^\frac2p\,\|f-f_m\|_{L^\infty(\P)}^{1-\frac2p} \leq C\,\|f\|_{\W^L}\,m^{-1/p}.
\]
\end{remark}

\begin{corollary}
For every compact set $K$ and $f\in \W^L(K)$, there exists a sequence of finite neural networks with $L$ hidden layers 
\[
f_n(x) = \sum_{i_L=1}^{m_{n,L}}a^{n,L}_{i_L}\sigma\left(\sum_{i_{L-1}=1}^{m_{n,L-1}} a^{n,L-1}_{i_Li_{L-1}}\sigma\left(\sum_{i_{L-2}} \dots \sigma\left(\sum_{i_1=1}^{m_{n,1}}a_{i_2i_1}^{n,1}\,\sigma\left(\sum_{i_0=1}^{d+1} a^{n,0}_{i_1i_0}\,x_{i_0}\right)\right)\right)\right)
\]
such that $\|f_n\|_{\W^L}\leq \|f\|_{\W^L}$ and $f_n\to f$ in $C^{0,\alpha}(K)$ for every $\alpha<1$.
\end{corollary}

\begin{proof}
We take $R>0$ such that $K\subseteq B_R(0)$ and take $\P$ to be the uniform distribution on $B_R(0)$. Fix $\eps>0$ and $\mu$ such that $f = f_\mu$ on $K$ and $\|\mu\|_{\M(B^{\B_{L-1}})} \leq \|f\|_{\W^L} + \eps$. Then we can approximate $f_\mu$ in $L^2(\P)$ by Theorem \ref{direct approximation theorem} with the norm bound $\|f_n\|_{\W^L} \leq  \|f\|_{\W^L} + \eps$.

By compactness, we find that $f_n$ converges to a limit in $C^{0,\alpha}(\overline{B_R(0)})$ for all $\alpha<1$, which coincides with the $L^2(\P)$-limit $f$. In particular, $f_n$ converges in $C^{0,\alpha}(K)$. We can eliminate the $\eps$ in the norm bound by a diagonal sequence argument.
\end{proof}

\begin{remark}
The direct and indirect approximation theorems show that neural tree spaces are the correct function spaces for neural networks under path-norm bounds. The construction of vector spaces and proofs made ample use of the equivalence between neural networks and neural trees. It is tempting to try to force more classical neural network structures by prescribing that the width of all layers tends to infinity at the same rate. However, this does not change the approximation spaces since in the direct approximation theorem, we can repeat a function from the approximating sequence multiple times until the width of the most restrictive layer is sufficiently large to pass to the next element in the sequence. A more successful approach is discussed in Section \ref{section indexed}.
\end{remark}

\subsection{Composition of multi-layer functions}
Let $f\in \big(\W^L(K))^k$ be an $L$-layer function with values in $\R^k$. Since $K$ is compact and $f$ is continuous,  $f(K)$ is also compact. Let $g\in \W^\ell(f(K))$ be an $\ell$-layer function on $f(K)$.

\begin{lemma}
$g\circ f\in \W^{L+\ell}(K)$ and
\begin{equation}
\big\|g\circ f\big\|_{\W^{L+\ell}(K)} \leq \|g\|_{\B^\ell(f(K))} \sup_{1\leq i\leq k} \|f_i\|_{\B^L(K)}.
\end{equation}
\end{lemma}

\begin{proof}
We proceed by induction. First consider the case $\ell=0$. Then $g(x) = w^Tf(x)$, so $w^T f(x) = \sum_{i=1}^k w_i\, f_i(x)$ is a (weighted) sum of $L$-layer functions, i.e.\ an $L$-layer function. By the triangle inequality we have
\[
\|g\circ f\|_{\W^L} \leq \sum_{i=1}^k |w_i|\,\|f_i\|_{\W^L}\leq \|w\|_{\ell^1} \sup_{1\leq i\leq k} \|f_i\|_{\W^L} = \|g\|_{\W^0}\sup_{1\leq i \leq k} \|f_i\|_{\W^L}
\]
Now assume that the theorem has been proved for $\ell-1$ with $\ell\geq 1$. To avoid double superscripts, denote by $B^\ell$ the closed unit ball in $\W^\ell(K)$. Let $g(z) = \int_{B^{\ell-1}} \sigma(h(z))\,\mu(\d h)$. Then
\begin{align*}
(g\circ f) &= \int_{B^{\ell-1}} \sigma\big((h\circ f)\big)\,\mu(\d h)\\
	&= \left(\sup_{1\leq i \leq k} \|f_i\|_{\B^L}\right) \int_{B^{\ell-1}} \sigma\left(\frac{h\circ f}{\sup_{1\leq i \leq k}\|f_i\|_{\B^L}}\right)\,\mu(\d h)\\
	&= \left(\sup_{1\leq i \leq k} \|f_i\|_{\B^L}\right) \int_{B^{L+\ell-1}} \sigma\left(j(\cdot)\right)\,(F_\sharp\mu)(\d j)
\end{align*} 
where 
\[
F: B^{{\ell-1}} \to B^{{L+\ell-1}}, \quad F(h) = \frac{h\circ f}{\sup_{1\leq i \leq k} \|f_i\|_{\W^L}}
\]
is well-defined by the induction hypothesis. By definition, $g\circ f\in \W^{L+\ell}$ with the appropriate norm bound.
\end{proof}

For generalized Barron spaces, we showed that $\|f\|_{X, K} \leq 2\, \|f\|_X$ for all $f\in X$, thus by induction $\|f\|_{\W^{\ell + L}} \leq 2^\ell\,\|f\|_{\W^L}$ for $L\geq 1$. We show that this naive bound can be improved to be independent of the number of additional layers.

\begin{lemma} 
Let $\ell, L\geq 1$ and $f\in \W^L(K)$. Then $f\in \W^{\ell+L}(K)$ and $\|f\|_{\W^{\ell+L}(K)}\leq 2\,\|f\|_{\W^L(K)}$.
\end{lemma}

\begin{proof}
Without loss of generality, $\|f\|_{\W^L(K)}\leq 1$. We note that $g_1=\sigma(f)$ and $g_2 = \sigma(-f)$ are both in the unit ball of $\W^{L+1}$ and non-negative, i.e.\ $g_1= \sigma(g_1)$ and $g_2 = \sigma(g_2)$. Thus $g_1, g_2$ are also in the unit ball of $\W^{L+2}$. By induction, we observe that $\|g_i\|_{\W^{L+\ell}(K)} \leq 1$ for all $\ell\geq 1$, $i=1,2$ and thus 
\begin{equation}
\|f\|_{\W^{L+\ell}(K)} = \|g_1 + g_2\|_{\W^{L+\ell}(K)} \leq \|g_1\|_{\W^{L+\ell}(K)} +\| g_2\|_{\W^{L+\ell}(K)}\leq 2.
\end{equation}
\end{proof}

\subsection{Rademacher complexity}
Considering statistical learning theory, neural tree spaces inherit the convenient properties of the space of affine functions. These convenient properties are one of the reasons why we study the path-norm in the first place. Recall the definition and discussion of Rademacher complexities from Section \ref{section Rademacher general}.

\begin{lemma}\label{lemma multi-layer complexity}
For every $L$, and every set of $N$ points $S\subseteq [-1,1]^d$, the hypothesis class $\H^L$ given by the closed unit ball in $\W^L$ satisfies the Rademacher complexity bound
\begin{equation}
\Rad\left(\H^{L};S\right) \leq 2^L\sqrt{\frac{2\,\log(2d+2)}N}.
\end{equation}
\end{lemma}

\begin{proof}
This follows directly from Example \ref{example rademacher linear} and Theorem \ref{theorem rademacher general} by induction.
\end{proof}

The complexity bound has an immediate application in statistical learning theory.

\begin{corollary}[Generalization gap]
Let $\mathcal \P$ be any probability distribution supported on $[-1,1]^d\times \R$ and $(X_1, Y_1) \dots, (X_N, Y_N)$ be iid random variables with law $\P$. Consider the hypothesis space $\H= \{h\in \W^L(K) : \|h\|_{\W^L(K)}\leq 1\}$. Assume that $\ell:\R\times \R \to [0, \bar c]$ is a bounded loss function. Then, with probability at least $1-\delta$ over the choice of data points $X_1,\dots, X_N$, the estimate
\begin{align}
\sup_{h\in \H} &\left|\frac1N\sum_{i=1} \ell\big(h(X_i), Y_i\big) - \int_{[-1,1]^d\times \R} \ell\big(h(x), y\big)\,\P(\d x\otimes \d y)\right|  \leq 2^{L+1}\, \sqrt{\frac{2\,\log(2d+2)}N} + \bar c\,\sqrt{\frac{2\,\log(2/\delta)}{N}}
\end{align}
holds.
\end{corollary}

\begin{proof}
This follows directly from Lemma \ref{lemma multi-layer complexity} and \cite[Theorem 26.5]{shalev2014understanding}.
\end{proof}

Thus it is easy to ``learn'' a multi-layer function with low path norm in the sense that a relatively small size of sample data points is sufficient to understand whether the function has low population risk or not. More sophisticated methods can provide dimension-dependent decay rates $1/2 + 1/{(2d+2)}$ of the generalization error at the expense of constants scaling like $\sqrt{d}$ instead of $\log(d)$ \cite[Remark 1]{barron2018approximation}.

\subsection{Generalization error estimates for regularized model}

As an application, we prove that empirical risk minimization with explicit regularization is a successful strategy in learning multi-layer functions. For technical reasons, we work with a bounded modification of $L^2$-risk instead of the mean squared error functional. 

Let $\P$ be a probability measure on $[-1,1]^d$ and $S= \{x_1,\dots, x_N\}$ be a set of samples drawn iid from $\P$. Denote
\[
\Risk, \Risk_n : \W^L\to \R, \qquad \Risk(f) = \int_{\R^d} \ell(x,f(x))\,\P(\d x), \qquad \Risk_N(f) = \frac1N\sum_{i=1}^N \ell(x_i, f(x_i)) 
\]
where the loss function $\ell$ satisfies
\[
\ell(x, y) \leq \min \big\{\bar c, |y-f^*(x)|^2\big\}.
\]
For finite neural networks with weights $(a^L, \dots, a^0)\in \R^{m_L} \times \dots \times \R^{m_1\times d}$ we denote
\begin{align*}
\widehat \Risk_N (a^L,\dots,a^0) &= \Risk_N(f_{a^L,\dots, a^0})\\
f_{a^L,\dots, a^0}(x) &= \sum_{i_L=1}^{m_L}a^L_{i_L}\sigma\left(\sum_{i_{L-1}=1}^{m_{L-1}} a^{L-1}_{i_Li_{L-1}}\sigma\left(\sum_{i_{L-2}} \dots \sigma\left(\sum_{i_1=1}^{m_1}a_{i_2i_1}^1\,\sigma\left(\sum_{i_0=1}^{d+1} a^0_{i_1i_0}\,x_{i_0}\right)\right)\right)\right).
\end{align*}

\begin{theorem}[Generalization error]\label{theorem generalization error}
Assume that the target function satisfies $f^*\in \W^L$. Let $\F_m$ be the class of neural networks with architecture like in Theorem \ref{direct approximation theorem}. The minimizer $f_m\in \F_m$ of the regularized risk functional
\[
\widehat \Risk_n(a^L,\dots, a^0) + \frac{9\,L^2}m \left[\sum_{i_L=1}^{m_L}\dots \sum_{i_0=1}^{d+1} \big| a^L_{i_L}\,a^{L-1}_{i_Li_{L-1}}\dots a^0_{i_1i_0}\big|\right]^2
\]
satisfies the risk bound
\begin{equation}\label{eq a priori risk bound} 
\Risk(f_m) \leq \frac{18\,L^2\,\|f^*\|_{\W^L}^2}m +  2^{L+3/2}\|f^*\|_{\W^L}\sqrt{\frac{2\,\log(2d+2)}N} + \bar c\,\sqrt{\frac{2\,\log(2/\delta)}{N}}.
\end{equation}
\end{theorem}

The first term comes from the direct approximation theorem. The explicit scaling in $L$ looks unproblematic, but recall that the network re	quires $\sim m^{2L-1}$ parameters. The second term stems from the Rademacher bound and is subject to the `curse of depth'. An improvement in either term would lead to better a priori estimates. The third term is purely probabilistic and unproblematic.

\begin{proof}[Proof of Theorem \ref{theorem generalization error}]
Denote $\lambda= \lambda_{m} = 9\,L^2\,m^{-1}$ and let $\hat f_m = f_{\hat a_L, \dots, \hat a_0}$ be like in Theorem \ref{direct approximation theorem}, i.e.\
\[
\|\hat f_m - f^*\|_{L^2(\P_n)} \leq \frac{3\,L\,\|f^*\|_{\W^L}}{\sqrt m}, \qquad \sum_{i_L, \dots, i_0} \big|a^L_{i_L}\dots a^0_{i_1i_0}\big|\leq \|f^*\|_{\W^L}.
\]
Then by definition
\begin{align*}
\widehat \Risk_n(a^L,\dots&, a^0) + \lambda \left[\sum_{i_L=1}^{m_L}\dots \sum_{i_0=1}^{d+1} \big| a^L_{i_L}\,a^{L-1}_{i_Li_{L-1}}\dots a^0_{i_1i_0}\big|\right]^2\\
&\leq 
\widehat\Risk_n(\hat a^L,\dots, \hat a^0) + \lambda \left[\sum_{i_L=1}^{m_L}\dots \sum_{i_0=1}^{d+1} \big| \hat a^L_{i_L}\,\hat a^{L-1}_{i_Li_{L-1}}\dots \hat a^0_{i_1i_0}\big|\right]^2\\
	&\leq \,\frac{9\,L^2\,\|f^*\|_{\W^L}^2}m + \lambda\|f^*\|_{\W^L}^2.
\end{align*}
In particular
\[
\|f_{a^L,\dots,a^0}\|_{\W^L}^2 \leq \left[\sum_{i_0=1}^{d+1} \big| a^L_{i_L}\,a^{L-1}_{i_Li_{L-1}}\dots a^0_{i_1i_0}\big| \right]^2\leq \frac{2\,\lambda\,\|f^*\|_{\W^L}^2}{\lambda} = 2\,\|f^*\|_{\W^L}^2.
\]
The Rademacher complexity is the supremum of linear random variables, so $\Rad(B_R^L;S) = R\cdot\Rad(B_1^L; S)$ where $B_R^L$ denotes the ball of radius $R$ centered at the origin in $\W^L$.
We conclude that, with probability at least $1-\delta$ over the draw of the training sample, we have
\begin{align*}
\Risk(f_{a^L,\dots,a^0}) &\leq \Risk_n(f_{a^L,\dots,a^0}) + 2^{L+3/2}\|f^*\|_{\W^L} \sqrt{\frac{2\,\log(2d+2)}N} + \bar c\,\sqrt{\frac{2\,\log(2/\delta)}{N}}\\
	&= \frac{18\,L^2\,\|f^*\|_{\W^L}^2}m +  2^{L+3/2}\|f^*\|_{\W^L}\sqrt{\frac{2\,\log(2d+2)}N} + \bar c\,\sqrt{\frac{2\,\log(2/\delta)}{N}}.
\end{align*}
\end{proof}

\begin{remark}
Since $\|f\|_{L^\infty} \leq \big(1+ \sup_{x\in K}|x|\big)\|f\|_{\W^L}$ for all $f\in \W^L(K)$, we can repeat the argument for the loss function $\ell(x,y) = |y-f^*(x)|^2$, which is a priori unbounded, but can be modified outside of the interval which $f_m, f^*$ take values in due to the a priori norm bound. The constant $\bar c$ in \eqref{eq a priori risk bound} in this case is
\[
\bar c = 4\,\|f^*\|_{\W^L([-1,1]^d)}^2.
\]
\end{remark}

\begin{remark}
For large $L$, these bounds degenerate rapidly. In \cite{barron2018approximation}, the authors show that under the stronger condition that a balanced version of the path-norm (which measures the average weights of incoming and outcoming paths at all nodes in all layers), a better bound on the Rademacher complexity is available. The balanced path norm achieves control over cancellations and the balancing of weights at different layers.

Heuristically, the proof proceeds as follows: Let $S=\{x_1,\dots, x_N\}$ be a sample set in $[-1,1]^d$ and the hypothesis space $\H$ be given by the unit ball in $\W^L$. By the direct approximation theorem, there exists a network with $O(m^{2L})$ weights which approximates $f$ with $\|f\|_{\W^L}\leq 1$ to accuracy $\sim \frac{L}{\sqrt m}$ in $L^2(\P_N)$ where $\P_N$ is the uniform measure on $S$. Thus the covering number $\overline N_{\eps, L^2(\P_N)}(\H)$ of $\H$ in the $L^2(\P_N)$-distance {\em should} scale like $L\eps^{-4L}$. 

Since $f(x_1),\dots f(x_N) \subset B_{\sqrt{m}}(0)\subseteq\R^N$ (with respect to the Euclidean distance), the Rademacher complexity can be bounded by 
\begin{align*}
\Rad(\H; S) &\leq \frac{2^{-K}\,\sqrt{m}}{\sqrt m} + \frac{6\sqrt{m}}m\,\sum_{i=1}^K2^{-i}\sqrt{ \log\big(\overline N_{2^{-i}\sqrt{m}, L^2(\P_N)}\big)}\\	
	&\leq 2^{-K} + \frac C{\sqrt m}\,\sum_{i=1}^K2^{-i} \sqrt{\log\big(L\,m^{-2L}\,2^{4iL})\big)}\\
	&\approx \frac{2^{-K}}{\sqrt m} + \frac {C\sqrt{L}}{\sqrt m}\,\sum_{i=1}^K 2^{-i}\sqrt{i}
\end{align*}
for any $K\in \N$ using \cite[Lemma 27.1]{shalev2014understanding}. Taking $K\to\infty$, only $\sqrt{L}$ enters in the estimate. The point in the proof that needs to be made rigorous is the connection between covering the parameter space with an $\eps$-fine net and covering the function class with an $\eps$-fine net. For a neural network
\[
f(x) = \eps^2\,\sigma\left(\frac1\eps x\right)
\]
the path-norm is bounded, but an $\eps$-small change in the outer layer would lead to a large change in the function space. Thus a balanced version of the path-norm is needed. In some cases, this may be possible through rescaling layers, but see Remark \ref{remark Hilbert weights vs path-norm} for a possible obstruction. Similar ideas are explored below in Section \ref{section Hilbert weights}, although we do not estimate the Rademacher complexity explicitly.

The ability to obtain Rademacher estimates from covering also suggests that improvements in the direct approximation theorem may not be possible, since the complexity of the function classes should increase with increasing depth.
\end{remark}

Unfortunately, the convenience in learning functions comes at a price when considering the approximation power of neural tree spaces as described in \cite[Corollary 3.4]{approximationarticle} for general function classes of low complexity.

\begin{corollary}\label{corollary kolmogorov width decay}
For any $d\geq 3$ exists a $1$-Lipschitz function $\phi$ on $[0,1]^d$ such that 
\begin{equation}
\limsup_{t\to\infty} \left(t^\gamma\,\inf_{\|f\|_X \leq t}\|\phi - f\|_{L^2(Q)}\right) = \infty.
\end{equation}
for all $\gamma> \frac{2}{d-2}$.
\end{corollary}

Thus to approximate even relatively regular functions in a fairly weak topology up to accuracy $\eps$, the path-norm of a network with $L$ hidden layers may have to grow (almost) as quickly as $\eps^{-\frac{d-2}2}$ independently of $L$. In particular, increasing the depth of an infinitely wide network does not increase the approximation power sufficiently to approximate general Lipschitz functions (while the path norm remains bounded by the same constant).

\subsection{Countably wide neural networks}

Let us briefly comment on another natural concept of infinitely wide neural networks. The space of countably wide networks
\[
f(x) = \sum_{i_L=1}^{\infty}a^L_{i_L}\sigma\left(\sum_{i_{L-1}=1}^{\infty} a^{L-1}_{i_Li_{L-1}}\sigma\left(\sum_{i_{L-2}} \dots \sigma\left(\sum_{i_1=1}^{\infty}a_{i_2i_1}^1\,\sigma\left(\sum_{i_0=1}^{d+1} a^0_{i_1i_0}\,x_{i_0}\right)\right)\right)\right)
\]
equipped with the path-norm 
\[
\|f\| = \inf_a \sum_{i_L} \dots\sum_{i_0} \big| a^L_{i_L}\dots a^0_{i_1i_0}\big|
\]
is a subspace of $\W^L$ by the same reasoning as Theorem \ref{embedding theorem} and the fact that the cross-product of a finite number of countable sets is countable. Like in the introduction, we can show that the spaces of countably wide neural networks and neural trees coincide.

Unlike finite neural networks, countable networks form a vector space. The space of countably wide networks is a proper subspace of $\W^L$ which contains all finite neural networks. The direct approximation theorem implies that the unit ball in the space of countably wide neural networks is not closed in weaker topologies like $C^{0,\alpha}$ or $L^p$. Thus the space of countably wide neural networks is not suitable from the perspective of variational analysis.

Intuitively, any convergent infinite sum contains a finite number of macroscopic terms and an infinite tail of rapidly decaying terms. Thus at initialization and throughout training, a scale difference would exist in a countable neural network between leading order neurons and tail neurons. This is not a useful way to think of neural networks where parameters in a fixed layer are typically chosen randomly from the  same distribution and then optimized by gradient flow-type algorithms. It should be noted however that common schemes like Xavier initialization \cite{glorot2010understanding} choose the weights in a fashion which makes the path-norm grows beyond all bounds as the number of neurons goes to infinity.

\section{Indexed representation of arbitrarily wide neural networks}\label{section indexed}

\subsection{Neural networks with general index sets}

The spaces considered above are a bit abstract. In this section, we discuss a more concrete representation for a subspace
of $\W^L$. As we show below, this subspace is invariant under the gradient flow dynamics.  For all practical purposes, it might just be the
right set of functions that we need to consider.

In \cite[Section 2.8]{barron_new}, we showed that $f:\R^d\to\R$ is a Barron function if and only if there exist measurable maps $a, b:(0,1)\to\R$ and $w:(0,1)\to \R^d$ such that
\[
f(x) = f_{a,w,b}(x) = \int_0^1 a_\theta\,\sigma\big(w_\theta^T x+ b_\theta\big)\,\d\theta.
\]
Furthermore
\[
\|f\|_{\B(K)} = \inf\left\{\int_0^1 |a|\,\big[|w| + |b|\big](\theta)\,\d\theta\:\bigg|\: f = f_{a,w,b} \text{ on }K\right\}.
\]
Thus we can think of Barron space as replacing the finite sum over neurons by an integral and replacing the index set $\{1,\dots,m\}$ by the (continuous) unit interval. We extend the approach to multi-layer networks in some generality.

\begin{definition}
For $0\leq i\leq L$, let $(\Omega_i, \A_i, \pi^i)$ be probability spaces where $\Omega_0 = \{0,\dots, d\}$ and $\pi^0$ is the normalized counting measure. Consider measurable functions $a^{L}:\Omega_L\to \R$ and $a^{i}:\Omega_{i+1}\times \Omega_i\to \R$ for $0\leq i \leq L-1$. Then define
{\footnotesize
\begin{equation}\label{eq network structure}
f_{a^L,\dots,a^0}(x) = \int_{\Omega_L} a^{(L)}_{\theta_L}\,\sigma\left(\int_{\Omega_{L-1}}\dots\sigma\left(\int_{\Omega_1} a^1_{\theta_2,\theta_1}\sigma\left(\int_{\Omega_0} a^0_{\theta_1,\theta_0}\,x_{\theta_0}\pi^0(\d\theta_0)\right)\pi^1(\d\theta_1)\right)\dots\,\pi^{(L-1)}(\d\theta_{L-1}) \right)\,\pi^L(\d\theta_L).
\end{equation}
}
Consider the norm
\begin{equation}
\|f\|_{\Omega_L,\dots,\Omega_0;K} = \inf \left\{\int_{\prod_{i=0}^L\Omega_i} \big| a^{(L)}_{\theta_L} \dots a^{(0)}_{\theta_1\theta_0}\big|\,\big(\pi^L\otimes\dots\otimes \pi^0\big)(\d\theta_L\otimes\dots\otimes \d\theta_0)\:\bigg|\: f= f_{a^L,\dots,a^0} \text{ on }K\right\}
\end{equation}
As usual, we set
\begin{equation}
X_{\Omega_L,\dots,\Omega_0;K} = \{ f\in C^{0,1}(K) : \|f\|_{\Omega_L,\dots,\Omega_0;K}<\infty\}.
\end{equation}
We call $X_{\Omega_L,\dots,\Omega_0;K}$ the {\em class of neural networks over $K$ modeled on the index spaces $\Omega_i = (\Omega_i, \A_i, \pi^i)$}.
\end{definition}

The representation in \eqref{eq network structure} can also be written as:
\begin{equation}
\label{eq: expectation}
    f(\bx)= \E_{\theta_L \sim \pi_L}  a^{(L)}_{\theta_L} \sigma(\E_{\theta_{L-1} \sim \pi_{L-1}}\dots
     \sigma(\E_{\theta_1 \sim \pi_1} a^1_{\theta_2,\theta_1}\sigma( 
         a^0_{\theta_1}\cdot  \bx) )\dots  )
\end{equation}
Representing functions as some form of expectations is the starting 
point for the continuous formulation of machine learning.

As we mentioned above, $\W^1(K) = X_{(0,1), \{0,\dots,d\}}$ where the unit interval is equipped with Lebesgue measure. The collection of {\em finite} neural networks is realized when all sigma-algebras $\A_i$ contain only finitely many sets (in particular, if all probability spaces are finite). In this situation, $X_{\Omega_L,\dots,\Omega_0}$ is not even a vector space.  

\begin{lemma}\label{lemma properties network functions}
\begin{enumerate}
\item For $L\geq 2$ and any selection of probability spaces $\Omega_L, \dots, \Omega_1$, the space of neuronal embeddings $X_{\Omega_L, \dots, \Omega_0;K}$ is a subset of the neural tree space $\W^L(K)$.
\item If $\Omega_i=(0,1)$ and $\pi^i$ is Lebesgue measure for all $i\geq 1$, then $X_{\Omega_L,\dots,\Omega_0; K}$ is a vector-space and $\|\cdot\|_{\Omega_L,\dots,\Omega_0;K}$ is a norm on it.

\item If $\Omega_i=(0,1)$ and $\pi^i$ is Lebesgue measure for all $i\geq 1$, then $X_{\Omega_L,\dots,\Omega_0; K}$ contains all finite neural networks with $L$ hidden layers. In particular, $X_{\Omega_L, \dots, \Omega_0; K}$ is a subspace of $\W^L(K)$ which is dense in $\W^L(K)$ with respect to the $C^{0,\alpha}$-topology for all $\alpha<1$ and consequently in $L^p(\P)$ for any probability measure on $K$, $p\in [1,\infty]$.

\item $X_{(0,1),(0,1),\{0,\dots, d\} ; K}$ contains Barron space $\W^1(K)$ and 
\[
\|f\|_{(0,1),(0,1),\{0,\dots, d\} ; K} \leq 2\,\|f\|_{\W^1(K)} \qquad \forall\ f\in \W^1(K).
\]

\item Let $f\in \big(\W^1(K))^k$ be a vector-valued Barron function and $g\in \W^1(\R^k)$ a scalar-valued Barron function. Then the composition $g\circ f$ lies in $X_{(0,1), (0,1), \{0,\dots,d\}; K}$ and
\begin{equation}
\|g\circ f\|_{(0,1),(0,1),\{0,\dots, d\} ; K} \leq \|g\|_{\W^1(\R^k)}\,\sum_{i=1}^k \|f_i\|_{\W^1(K)}.
\end{equation}
In particular, this includes
\begin{itemize}
\item the absolute value/positive part/negative of a Barron function
\item the pointwise maximum/minimum of two Barron functions,
\item the product of two Barron functions.
\end{itemize}
\end{enumerate}
\end{lemma}

\begin{proof}
{\bf First claim.} This can be proved exactly like Theorem \ref{embedding theorem}.

{\bf Second claim.} For any choice of parameter spaces $\Omega_i$, the set of functions $X_{\Omega_L, \dots, \Omega_0; K}$ is a balanced cone, i.e.\ if $f\in X_{\Omega_L, \dots, \Omega_0; K}$ then $\lambda f\in X_{\Omega_L, \dots, \Omega_0; K}$ for all $\lambda\in \R$. It remains to show that $X_{\Omega_L, \dots, \Omega_0; K}$ is closed under function addition. Let 
\begin{align*}
f(x) &= \int_0^1 a^L_{\theta_L} \,\sigma\left(\int_0^1\dots \sigma\left(\int_0^1 a^1_{\theta_2\theta_1}\,\sigma\left(\frac1{d+1}\sum_{\theta_0=1}^{d+1}a^0_{\theta_1\theta_0}x_{\theta_0}\right)\right)\right)\\
g(x) &= \int_0^1 b^L_{\theta_L} \,\sigma\left(\int_0^1\dots \sigma\left(\int_0^1 b^1_{\theta_2\theta_1}\,\sigma\left(\frac1{d+1}\sum_{\theta_0=1}^{d+1}b^0_{\theta_1\theta_0}x_{\theta_0}\right)\right)\right).
\end{align*}
Then
\begin{align*}
(f+g)(x) &= \int_0^1 c^L_{\theta_L} \,\sigma\left(\int_0^1\dots \sigma\left(\int_0^1 c^1_{\theta_2\theta_1}\,\sigma\left(\frac1{d+1}\sum_{\theta_0=1}^{d+1}c^0_{\theta_1\theta_0}x_{\theta_0}\right)\right)\right)\\
\end{align*}
where
\[
c^L_\theta = \begin{cases} 2\,a^L_{2\theta}&\theta\in(0,1/2)\\ 2\,b^L_{2\theta-1} &\theta\in (1/2,1)\end{cases}, \qquad c^\ell_{\theta\xi} = \begin{cases} 4\, a^\ell_{2\theta,2\xi} &\theta, \xi \in (0, 1/2)\\ 4\,b^\ell_{2\theta-1,2\xi-1} &\theta,\xi \in (1/2,1)\\ 0&\text{else}\end{cases}.
\]
Essentially, we construct two parallel networks that are added in the final layer and otherwise do not interact. The pre-factors stem from the fact that we re-arrange a mean-field index set and could be eliminated if we chose more general measure spaces (e.g. $\Z$ or $\R$) as index sets.

{\bf Third claim.}  Any finite neural network can be written as a mean field neural network 
\[
f(x) = \frac1{m_L}\sum_{i_L=1}^{m_L}a^L_{i_L}\sigma\left(\frac1{m_{L-1}}\sum_{i_{L-1}=1}^{m_{L-1}} a^{L-1}_{i_Li_{L-1}}\sigma\left( \dots \sigma\left(\frac1{m_1}\sum_{i_1=1}^{m_1}a_{i_2i_1}^1\,\sigma\left(\frac1{d+1}\sum_{i_0=1}^{d+1} a^0_{i_1i_0}\,x_{i_0}\right)\right)\right)\right)
\]
Define the functions 
\begin{align*}
a^L:&(0,1)\to \R, \:\quad \:a^L(s) = a^L_i \:\:\:\quad\text{ for } \frac{i-1}{m_L}\leq s < \frac i{m_L}\\
a^\ell:&(0,1)^2\to \R,\quad a^\ell(r,s) = a^\ell_{ij}\quad\text{ for }\frac{i-1}{m_{\ell+1}}\leq r < \frac i{m_L}, \:\:\frac{j-1}{m_\ell} \leq s < \frac{i}{m_\ell}.
\end{align*}
for $0\leq \ell < L$. Then $f= f_{a^L,\dots, a^0}$. 

{\bf Fourth claim.} Let $f$ be a Barron function. Then, according to \cite[Section 2.8]{barron_new}, $f$ can be written as
\[
f(x) = \int_0^1 \bar a^1_\theta\,\sigma\left(\frac1{d+1}\sum_{i=1}^{d+1} a^0_{\theta,i}\,x_i\right)\,\d\theta
\]
For $\bar a^1, a^0 \in L^2(0,1)$. In particular,
\[
f(x) = \int_0^1 a^2_{\theta_2}\,\sigma\left(\int_0^1 a^1_{\theta_2\theta_1} \,\sigma\left(\frac1{d+1}\sum_{i=1}^{d+1}a^0_{\theta_1,i}\,x_i\right)\,\d\theta_1\right)\d\theta_2
\]
where
\[
a^2_{\theta_2} = \begin{cases} 2 & \theta_2<1/2\\ -2 &\theta_2 >1/2\end{cases}, \qquad a^1_{\theta_2\theta_1} = \begin{cases} \bar a_{\theta_1} & \theta_2 < 1/2\\ -\bar a_{\theta_1}& \theta_2>1/2\end{cases}.
\]

{\bf Fifth claim.} Let $f_{k+1}\equiv 1$. For $1\leq i\leq k$, let
\begin{align*}
f_i(x) &= \int_0^1 a^i_s\,\sigma\left(\frac1{d+1}\sum_{j=1}^{d+1}b^i_{s,j}x_j\right)\,\d s\\
g (y) &= \int_0^1 c_t \,\sigma\left(\frac1{k+1}\sum_{l=1}^{k+1} d_{t,l}y_l\right)\,\d t.
\end{align*}
Then 
\begin{align*}
(g\circ f)(x) &= \int_0^1 c_t \,\sigma\left(\frac1{k+1}\sum_{i=1}^{k+1} d_{t,i}\int_0^1 a^i_s\,\sigma\left(\frac1{d+1}\sum_{j=1}^{d+1}b^i_{s,j}x_j\right)\,\d s\right)\,\d t\\
	&= \int_0^1 c_t \,\sigma\left(\int_0^1 \bar a_{ts}\,\sigma\left(\frac1{d+1}\sum_{j=1}^{d+1}\bar b_{s,j}x_j\right)\,\d s\right)\,\d t
\end{align*}
where
\[
\bar a_{ts} = d_{t,i}\,a^i_{(k+1)(s-\frac{i-1}{k+1})} \text{ for }\frac{i-1}{k+1} \leq s \leq \frac i{k+1}, \qquad \bar b _{s,j} = b^i_{(k+1)(s-\frac{i-1}{k+1}), j}\text{ for }\frac{i-1}{k+1}.
\]
For the special cases observe that
\[
g(z) = \sigma(z), \qquad g(z_1, z_2) = \max\{z_1, z_2\} = z_1 + \sigma(z_2-z_1)
\]
are Barron functions, thus the first two claims are immediate. Furthermore
\[
\tilde g(z) = \max\{0, z\}^2 = \int_\R 1_{0,\infty} 2\,\sigma(z-\xi)\,\d\xi
\]
is a Barron function on bounded intervals, and so is $z\mapsto z^2$. The Barron functions $f_1, f_2$ are continuous on a compact set $K$ and hence bounded. It follows that
\[
f_1f_2 = \frac14\left[\big(f_1 + f_2\big)^2 - \big(f_1 - f_2\big)^2\right] \in X_{(0,1), (0,1),\{0,\dots, d\}; K}
\] 
\end{proof}

In particular, $X_{(0,1), (0,1), \{0,\dots, d\}; K}$ contains many functions which are not in Barron space (compare \cite[Remark 5.12]{barron_new}).

\begin{remark}
The unit interval with Lebesgue measure is a probability space with two convenient properties for our purposes:

\begin{enumerate}
\item For any finite collection of numbers $0\leq \alpha_1,\dots,\alpha_N\leq 1$ such that $\sum_{i=1}^N \alpha_i=1$, there exist disjoint measurable subsets $I_i\subseteq(0,1)$ such that $\L(I_i) = \alpha_i$ for all $1\leq i\leq N$. This allows us to embed finite networks of arbitrary width (and can be extended to countable sums).

\item There exist measurable bijections between the unit interval and many index sets which appear larger at first sight. By rearranging decimal representations, we may for example construct a measurable bijection between $(0,1)$ and $(0,1)^d$ for any $d\geq 1$. Using a hyperbolic tangent or similar for rescaling, we can further show that a measurable bijection between $(0,1)$ and $\R^d$ exists. Furthermore, using the characteristic function of a probability measure $\pi$ on $(0,1)$, we can find a measurable map $\phi:(0,1)\to (0,1)$ such that $\phi_\sharp \pi$ is Lebesgue measure. For details, see e.g.\ \cite[Section 2.8]{barron_new}.
\end{enumerate}

The entire analysis remains valid for any index set with these two properties. We describe a more natural (but also more complicated) approach in Section \ref{section natural index set}.
\end{remark}

\begin{remark}\label{remark measure isomorphism}
Let $(\Omega_\ell, \A_\ell, \pi^\ell)$, $(\widetilde \Omega_\ell, \widetilde\A_\ell, \widetilde\pi^\ell)$ be families of probability spaces for $0\leq \ell\leq \Omega_L$ and $\phi^\ell:\Omega_\ell\to \widetilde\Omega_\ell$ measurable maps such that $\widetilde\pi^\ell = \phi^\ell_\sharp \pi^\ell$. Then the spaces
\[
X_{\Omega_L,\dots,\Omega_0;K} = \W_{\pi^L,\dots,\pi^0}(K)
\]
coincide and the norms induced by network representations with the different index spaces agree.
\end{remark}

\begin{remark}
We never used that the measures $\pi^i$ are probability measures (or even finite). More general measures could be used on the index set. In particular, the analysis of this section also applies to the space of countably wide neural networks (which corresponds to the integers with the counting measure).

There currently seems little gain in pursuing that generality, and we will remain in the natural mean field setting of networks indexed by probability spaces.
\end{remark}

\subsection{Networks with Hilbert weights}\label{section Hilbert weights}

We can bound the path-norm by a more convenient expression. At first glance, it looks as though the weight functions $a^i$ are required to satisfy a restrictive integrability condition like $a^i \in L^{L+1}(\pi^L\otimes\dots\otimes\pi^0)$. This can be weakened significantly by using the neural-network structure in which indices for layers separated by one intermediate layer are independent.

\begin{lemma}\label{lemma path-norm bound}
For any $f$, the path-norm is bounded by 
\begin{align}
\|f\|_{\Omega_L,\dots,\Omega_0;K} \leq \inf \left\{\|a^L\|_{L^2(\pi^L)}\,\prod_{i=0}^{L-1} \|a^i\|_{L^2(\pi^{i+1}\otimes\pi^i)} \:\bigg|\: a^i\text{ s.t. }f= f_{a^L,\dots,a^0} \text{ on }K\right\}
\end{align}
\end{lemma}

\begin{proof}
To simplify notation, we denote $\pi^i(\d\theta_i) = \d\theta_i$ as in the case of the unit interval. The proof goes through in the general case. We quickly observe that for a network with two hidden layers, we can easily bound
\begin{align*}
\int_{\Omega_2\times\Omega_1\times\Omega_0} \big|a^2_{\theta_2}& \,a^1_{\theta_2\theta_1}\,a^0_{\theta_1\theta_0}\big| \d\theta_2\,\d\theta_1\,\d\theta_0 
	= \int_{\Omega_2\times\Omega_1\times\Omega_0} \big|a^2_{\theta_2}a^0_{\theta_1\theta_0}\big| \,\big|a^1_{\theta_2\theta_1}\,\big| \d\theta_2\,\d\theta_1\,\d\theta_0\\
	&\leq \left(\int_{\Omega_2\times\Omega_1\times\Omega_0} \big|a^2_{\theta_2}a^0_{\theta_1\theta_0}\big|^2 \,\d\theta_2\,\d\theta_1\,\d\theta_0\right)^\frac12 \left(\int_{\Omega_2\times\Omega_1\times\Omega_0}  \big|a^1_{\theta_2\theta_1}\,\big|^2\,\d\theta_2\,\d\theta_1\,\d\theta_0\right)^\frac12\\
	&= \left(\int_{\Omega_2} \big|a^2_{\theta_2}\big|^2\,\d\theta_2\right)^\frac12 \left(\int_{\Omega_2\times\Omega_1} \big|a^1_{\theta_2\theta_1}\big|^2\d\theta_2\d\theta_1\right)^\frac12 \left(\int_{\Omega_1\times\Omega_0} \big|a^0_{\theta_1\theta_0}\big|^2\,\d\theta_1\d\theta_0\right)^\frac12
\end{align*}
In the general case, we set $\Omega_{L+1} = \{0\}$ to simplify notation. the argument follows as above by
\begin{align*}
\|f_{a^L, \dots, a^0}&\|_{\Omega_L,\dots,\Omega_0;K} \leq \int_{\prod_{i=0}^{L+1}\Omega_i} \big| a^{(L)}_{\theta_{L+1}\theta_L} \dots a^{(0)}_{\theta_1\theta_0}\big|\,\d\theta_L\dots\d\theta_0\\
	&= \int_{\prod_{i=0}^{L+1}\Omega_i} \left|\prod_{i=0}^{\lfloor L/2\rfloor } a^{2i}_{\theta_{2i+1}\theta_{2i}}\right|\, \left|\prod_{i=0}^{\lfloor (L-1)/2\rfloor } a^{2i+1}_{\theta_{2i+2}\theta_{2i+1}}\right|\,\d\theta_L\dots\d\theta_0\\
	&\leq \left(\int_{\prod_{i=0}^{L+1}\Omega_i} \left|\prod_{i=0}^{\lfloor L/2\rfloor } a^{2i}_{\theta_{2i+1}\theta_{2i}}\right|^2\,\d\theta_L\dots\d\theta_0\right)^\frac12 \left(\int_{\prod_{i=0}^{L+1}\Omega_i} \left|\prod_{i=0}^{\lfloor L-1/2\rfloor } a^{2i+1}_{\theta_{2i+2}\theta_{2i+1}}\right|^2|\,\d\theta_L\dots\d\theta_0\right)^\frac12\\
	&= \|a^L\|_{L^2(\pi^L)}\,\prod_{i=0}^{L-1} \|a^i\|_{L^2(\pi^{i+1}\times\pi^i)}.
\end{align*}
We may now take the infimum over all coefficient functions.
\end{proof}

The lemma allows us to analyze networks in a convenient fashion using only $L^2$-norms. In numerical simulations, explicit regularization by penalizing $L^2$-norms provides a smoother alternative to penalizing the path-norm directly. Note that the proof is built on the network index structure and does not extend to neural trees.

\begin{lemma}
The realization map
\begin{equation}\label{eq realization map}
F: L^2(\pi^{L})\times L^2(\pi^L\otimes \pi^{L-1}) \dots \times L^2(\pi^1\otimes \pi^0) \to C^{0}(K), \qquad F(a_L, \dots, a_0) = f_{a^L,\dots,a^0}
\end{equation}
is locally Lipschitz-continuous.
\end{lemma}

\begin{proof}
For $L=1$ and $x\in K$, note that
\begin{align*}
\big|f_{a^1,a^0}(x)& - f_{\bar a^1, \bar a^0}(x)\big| = \left|\int_{\Omega_1} a^1_{\theta_1} \sigma\left(\frac1{d+1}\sum_{\theta_0=1}^{d+1}a^0_{\theta_1\theta_0}x_{\theta_0}\right) - \bar a^1_{\theta_1} \sigma\left(\frac1{d+1}\sum_{\theta_0=1}^{d+1}\bar a^0_{\theta_1\theta_0}x_{\theta_0}\right) \,\pi^1(\d\theta_1)\right|\\
	&\leq \int_{\Omega_1} \big|a^1_{\theta_1} - \bar a^1_{\theta_1}\big|\left|\sigma\left(\frac1{d+1}\sum_{\theta_0=1}^{d+1}a^0_{\theta_1\theta_0}x_{\theta_0}\right)\right|\\
		&\qquad + \big|\bar a^1_{\theta_1}\big|\left|\sigma\left(\frac1{d+1}\sum_{\theta_0=1}^{d+1}a^0_{\theta_1\theta_0}x_{\theta_0}\right) -\sigma\left(\frac1{d+1}\sum_{\theta_0=1}^{d+1}\bar a^0_{\theta_1\theta_0}x_{\theta_0}\right)\right|\,\pi^1(\d\theta_1)\\
	&\leq \|a^1-\bar a^1\|_{L^2(\Omega_1)} \|a^0\|_{L^2(\Omega_1\times\Omega_0)}\sup_{x\in K}|x| + \|\bar a^1\|_{L^2(\Omega_1)}\|a^0 - \bar a^0\|_{L^2(\Omega_1\times\Omega_0)}\sup_{x\in K}|x|.
\end{align*}
The general case follows analogously by induction.
\end{proof}

We define a third class of spaces for the $L^2$-approach.

\begin{definition}
For $0\leq i\leq L$, let $(\Omega_i, \A_i, \pi^i)$ be a probability space where $\Omega_0 = \{0,\dots, d\}$ and $\pi^0$ is the normalized counting measure. Let $a^{L} \in L^2(\pi^L)$ and $a^{i}\in L^2(\pi^{i+1}\otimes \pi^i)$ for $0\leq i \leq L-1$. Then define like in \eqref{eq network structure}
{\footnotesize
\begin{align*}
f_{a^L,\dots,a^0}(x) &= \int_{\Omega_L} a^{(L)}_{\theta_L}\,\sigma\left(\int_{\Omega_{L-1}}\dots\sigma\left(\int_{\Omega_1} a^1_{\theta_2,\theta_1}\sigma\left(\int_{\Omega_0} a^0_{\theta_1,\theta_0}\,x_{\theta_0}\pi^0(\d\theta_0)\right)\pi^1(\d\theta_1)\right)\dots\,\pi^{(L-1)}(\d\theta_{L-1}) \right)\,\pi^L(\d\theta_L).
\end{align*}
}
We define the  {\em class of neural networks over $K$ with Hilbert weights over the index spaces $\Omega_i = (\Omega_i, \A_i, \pi^i)$} as the image of $L^2(\pi^L)\times\dots\times L^2(\pi^1\otimes \pi^0)$ under the realization map \eqref{eq realization map} and denote it by 
\[
\W_{\pi^L,\dots,\pi^0}(K) = \left\{f:K\to \R \:\bigg|\:\exists\ a^L\in L^2(\pi^L), \: a^\ell \in L^2(\pi^\ell\otimes \pi^\ell) \text{ s.t. } f\equiv f_{a^L,\dots,a^0} \text{ on }K\right\}.
\]
The function class is equipped with the measure of complexity
\begin{equation}
Q_{\pi^L,\dots,\pi^0; K}(f)  = \inf\left\{\|a^L\|_{L^2(\pi^L)}\,\prod_{i=0}^{L-1} \|a^i\|_{L^2(\pi^{i+1}\otimes\pi^i)} \:\bigg|\: a^i\text{ s.t. }f= f_{a^L,\dots,a^0} \text{ on }K\right\}.
\end{equation}
\end{definition}

We declare a notion of convergence on $\W_{\pi^L,\dots,\pi^0}(K)$ by the convergence of the weight functions in the $L^2$-strong topology. To avoid pathological cases, we normalize the weights across layers.
Using the homogeneity of $\sigma$, note that $f_{a^L,\dots,a^0} = f_{\lambda_\ell a^\ell, \dots, \lambda_0a^0} \in \W_{\pi^L,\dots,\pi^0}(K)$ for $\lambda_i>0$ such that $\prod_{i=0}^L \lambda_i =1$. In particular, we may assume without loss of generality that
\[
\|a^\ell\|_{L^2} = \left(\prod_{i=0}^L\|a^i\|_{L^2}\right)^\frac1{L+1}
\]
for all $\ell\geq 1$. 

\begin{definition}\label{definition metric hilbert weights}
We say that a sequence of functions $f_n \in \W_{\pi^L,\dots,\pi^0}(K)$ converges weakly to a limit $f\in \W_{\pi^L,\dots,\pi^0}(K)$ if there exist coefficient functions $a^{L, n},\dots, a^{0,n}$ for $n\in\N$ and $a^L,\dots, a^0$ such that
\begin{enumerate}
\item $f_n = f_{a^{L, n}, \dots, a^{0,n}}$ for all $n\in \N$ and $f= f_{a^L,\dots, a^0}$.
\item $\|a^{\ell, n}\| = \left(\prod_{i=0}^L\|a^{i,n}\|_{L^2}\right)^\frac1{L+1}$ for all $n\in \N$ and $0\leq \ell \leq L$. 
\item $\limsup_{n\to \infty} \left[\prod_{i=0}^L\|a^{i,n}\|_{L^2} - Q(f_n) \right] = 0$.
\item $a^{\ell, n}\to a^\ell$ in the $L^2$-strong topology for all $0\leq \ell\leq n$.
\end{enumerate}
\end{definition}
To evaluate the notion of convergence, consider the case $L=1$ and write $(a,w)$ for $(a^1, a^0)$. We interpret $a^0$ as an $\R^{d+1}$-valued function on $(0,1)$ rather than a scalar function on $(0,1)\times \{0,\dots, d\}$. Then it is easy to see that 
\[
(a^n, w^n)\to (a,w) \quad\text{strongly in }L^2(0,1)\qquad \Ra\qquad (a^n,w^n)_\sharp \L \to (a,w)_\sharp\L \quad\text{in Wasserstein}.
\] 
The inverse statement holds up to a rearrangement of the index set. The Wasserstein distance is associated with the weak convergence of measures, while the topology of Barron space is associated with the the norm topology for the total variation norm (strong convergence). This justifies the terminology of `weak convergence' of arbitrarily wide neural networks.

Weak convergence is locally metrizable, but not induced by a norm. A relaxed version of convergence described above is metrizable by the distance function
\begin{align*} 
d_{HW}(f,g) = \inf\Bigg\{\sum_{\ell=0}^L \|a^{\ell, f} - a^{\ell, g}\|_{L^2(\pi^\ell)}\,\bigg|&\,a^{L, f},\dots, a^{0,g} \text{ s.t. }f = f_{a^{L,f}, \dots, a^{0, f}}, \:g = f_{a^{L,g},\dots, a^{0,g}}\text{ and}\\ \label{eq hilbert weight metric} \showlabel
	&\quad\|a^{\ell, h}\| \equiv \left(\prod_{i=0}^L\|a^{i,h}\|_{L^2}\right)^\frac1{L+1} \leq 2\,Q(h)^\frac1{L+1}\text{ for }h\in \{f,g\}\Bigg\}.
\end{align*}
The third condition has been weakened from $\prod_{i=0}^L\|a^{i,n}\|_{L^2} - Q(f_n) \to 0$ to $Q(f_n)\leq \prod_{i=0}^L\|a^{i,n}\|_{L^2} \leq 2\, Q(f_n)$. The normalization is required to ensure that functions in which one layer can be chosen identical do not have zero distance by shifting all weight to the one layer. Which mode of convergence is superior to another remains to be seen. Equipped with the Hilbert weight metric $d_{HW}$, the spaces $\W_{\pi^L,\dots,\pi^0}(K)$ are complete.

To avoid the unwieldy terminology of arbitrarily wide neural networks with Hilbert weights, we introduce the following simpler terminology.

\begin{definition}\label{definition multi-layer space new}
The metric spaces $\W_{\pi^L,\dots,\pi^0}(K), d)$ equipped with the metric $d_{HW}$ from \eqref{eq hilbert weight metric} are called {\em multi-layer spaces} for short.
\end{definition}

\begin{remark}
As seen in Lemma \ref{lemma path-norm bound}, the inclusions
\begin{equation}\label{eq space inclusions}
\W_{\pi^L,\dots,\pi^0}(K) \quad\subseteq\quad X_{\Omega^L,\dots,\Omega^0; K} \quad\subseteq\quad \W^L(K)
\end{equation}
hold. The last three points of Lemma \ref{lemma properties network functions} hold with $\W_{\pi^L,\dots,\pi^0}(K)$ in place of $X_{\Omega^L,\dots,\Omega^0; K}$. We note however that the functions
\[
c^L_\theta = \begin{cases} 2\,a^L_{2\theta}&\theta\in(0,1/2)\\ 2\,b^L_{2\theta-1} &\theta\in (1/2,1)\end{cases}, \qquad c^\ell_{\theta\xi} = \begin{cases} 4\, a^\ell_{2\theta,2\xi} &\theta, \xi \in (0, 1/2)\\ 4\,b^\ell_{2\theta-1,2\xi-1} &\theta,\xi \in (1/2,1)\\ 0&\text{else}\end{cases}
\]
satisfy
\[
\|c^L\|_{L^2(0,1)}^2 = 2\,\left[\|a^L\|_{L^2(0,1)}^2 + \|b^L\|_{L^2(0,1)}^2\right], \qquad \|c^\ell\|_{L^2\big((0,1)^2\big)} = 4\,\left[\|a^\ell\|_{L^2\big((0,1)^2\big)} + \|b^\ell\|_{L^2\big((0,1)^2\big)}\right].
\]
In particular, if $\Omega^\ell= (0,1)$ and $\pi^\ell$ is Lebesgue measure for all $1\leq \ell\leq L$, then $\W_{\pi^L,\dots,\pi^0}(K)$ is a linear space, but both $Q_{\pi^L,\dots,\pi^0; K}$ and $d_{HW}$ generally fail to be a norm.
\end{remark}

\begin{remark}\label{remark Hilbert weights vs path-norm}
It is not clear whether the inclusions in \eqref{eq space inclusions} are necessarily strict. In the case of Barron space, it is easily possible to normalize by replacing
\[
a^1_{\theta_1} \mapsto \frac{a^1_{\theta_1}}{\rho_{\theta_1}}, \qquad a^0_{\theta_1\theta_0} \mapsto \rho_{\theta_1}\,a^0_{\theta_1\theta_0}
\]
such that both layers have the same magnitude in $L^2(0,1)$, even if they are only assumed to be measurable with finite path-norm a priori. For multiple layers, this may not be possible. Let
\begin{equation}
a_s \equiv 1,\qquad b_{st} = f(s-t), \qquad c_t\equiv 1
\end{equation}
where $f$ is a one-periodic function on $\R$ which is in $L^1(0,1)$, but not $L^2(0,1)$. Then any normalization
\[
a_s \mapsto \frac{a_s}{\rho_s}, \qquad b_{st}\mapsto \rho_s\,\tilde \rho_t\,b_{st}, \qquad c_t\mapsto \frac{c_t}{\tilde \rho_t}
\]
fails to make $b$ $L^2$-integrable. Whether or not this can be compensated by choosing other weights with the same realization remains an open question. \end{remark}

\subsection{Networks with two hidden layers}

We investigate the space $X_{(0,1), (0,1), \{0,\dots, d\}; K}$ and $\W_{\L^1,\L^1,\pi^0}(K)$ more closely where $\pi^0$ denotes counting measure and $\L^1$ is the Lebesgue measure on $(0,1)$. In general, any network modelled on probability spaces $\Omega_2, \Omega_1, \Omega_0$ can be written as
\begin{align*}
f(x) &= \int_{\Omega_2} a^{2}_{\theta_2}\,\sigma\left(\int_{\Omega_1} a^1_{\theta_2, \theta_1}\,\sigma\left(\sum_{\theta_0=1}^{d+1}a^0_{\theta_1,\theta_0}x_{\theta_0}\right)\,\pi^1(\d\theta_1)\right)\,\pi^2(\d\theta_2)\\
	&= \int_{\Omega_2} a^{2}_{\theta_2}\rho_{\theta_2}\,\sigma\left(\int_{\Omega_1} \frac{a^1_{\theta_2, \theta_1}|w_{\theta_1}|}{\rho_{\theta_2}}\,\sigma\left(\frac{w_{\theta_1}^T}{|w_{\theta_1}|} (x,1)\right)\,\pi^1(\d\theta_1)\right)\,\pi^2(\d\theta_2)\\
	&= \int_{\Omega_2} a^{2}_{\theta_2}\,\rho_{\theta_2}\,\sigma\left(\int_{\R\times S^d} \tilde a\,\sigma(\tilde w^Tx)\,(\Psi(\theta_2,\cdot)_\sharp \pi^1)(\d \tilde a\otimes \d \tilde w)\right)\pi^2(\d\theta_2)
\end{align*}
where $w_\theta = (a^{0}_{\theta_1,1},\dots, a^{0}_{\theta_1,d+1})$ and 
\[
\Psi:\Omega_2\times \Omega_1 \to \R \times S^d, \qquad \Psi(\theta_2, \theta_1) = \left(\frac{a^{1}_{\theta_2\theta_1}|w_{\theta_1}|}{\rho_{\theta_2}}, \frac{w_{\theta_1}}{|w_{\theta_1}|}\right).
\]
Since the second component of $\Psi$ does not depend on $\theta_2$, the marginal $\overline \pi$ of $\Psi(\theta_2,\cdot)_\sharp \pi^1$ on the sphere is independent of $\theta_2$. We can therefore write
\[
\int_{\R\times S^d} \tilde a\,\sigma(\tilde w^Tx)\,(\Psi(\theta_2,\cdot)_\sharp \pi^1)(\d \tilde a\otimes \d \tilde w) = \int_{S^d} \bar a^{\theta_2}(w)\,\sigma(w^Tx)\,\bar \pi(\d w)
\]
by integrating in the $a$-direction and making $\bar a$ a function of $w$ (see \cite[Section 2.3]{barron_new} for the technical details). Thus
\begin{align*}
f(x) &= \int_{\Omega_2} a^{2}_{\theta_2}\,\rho_{\theta_2}\,\sigma\left(\int_{S^d} \bar a^{\theta_2}(w)\,\sigma(w^Tx)\,\bar\pi(\d w)\right)\pi^2(\d\theta_2).
\end{align*}
We can in particular choose $\rho\geq 0$ such that 
\[
\int_{S^d} \big| \bar a^{\theta_2}(w)\big| \,\bar\pi(\d w) \leq \int_{\Omega^1} \frac{|a^{1}_{\theta_2\theta_1}|\,|w_{\theta_1}|}{\rho_{\theta_2}}\,\pi^1(\d\theta_1) = \frac 1{\rho_{\theta_2}}\int_{\Omega^1} {|a^{1}_{\theta_2\theta_1}|\,|w_{\theta_1}|}\,\pi^1(\d\theta_1)  \leq 1
\]
for all $\theta_2\in \Omega_2$. Then the map
\[
F: \Omega_2\to B^X, \qquad \theta_2\mapsto f^{\theta_2} = \int_{S^d} \bar a^{\theta_2}(w)\,\sigma(w^Tx)\,\bar\pi(\d w)
\]
is well-defined and Bochner integrable. In particular
\begin{align*}
f(x) &= \int_{\Omega_2} a^{(2)}_{\theta_2}\,\rho_{\theta_2}\,\sigma\big(f^{\theta_2}(x)\big)\,\pi^2(\d\theta_2) \\
	&= \int_{B^X} \sigma(g(x))\,\mu(\d g)
\end{align*}
where $\mu = F_\sharp \big((a^{(2)}\rho)\cdot \pi^2\big)$. By construction, $\mu$ is concentrated on the subspace $Y_{\bar\pi}$ of Barron functions which can be represented with an $L^1$-density with respect to the measure $\bar \pi$, by which we mean that $|\mu|(B^X\setminus Y_{\bar\pi}) =0$. This equation can be sensibly interpreted since any measure can be extended to a potentially larger $\sigma$-algebra containing all null sets. 

Thus general functions in $\W^2(K)$ and $X_{(0,1), (0,1),\{0,\dots,d\}; K}$ both take the form
\[
f(x) = \int_{B^X} \sigma(g(x))\,\mu(\d g)
\]
where $B^X$ is the unit ball in Barron space, but in the second case, $\mu$ is concentrated on a subspace $Y_{\bar\pi}$. This space is a quotient of $L^1(\bar\pi)$ by a closed subspace and thus closed in Barron space, but may be dense in $C^0(K)$. If $\bar\pi$ is the uniform distribution on $S^d$, then $Y_{\bar\pi}$ is dense in $C^0$ since $L^1(\bar\pi)$ is dense in the space of Radon measures on $S^d$ with respect to the weak topology.

{\em Claim:} There is no distribution $\bar\pi$ on $S^d$ such that every Barron function can be expressed with an $L^1$-density with respect to $\bar\pi$ if $K$ is the closure of an open set. 

{\em Proof of claim:} Barron space is not separable since
\[
\|\sigma(w_1^T\cdot) - \sigma(w_2^T\cdot)\|_{\B^1(K)} \geq [\sigma(w_1^T\cdot) - \sigma(w_2^T\cdot)]_{C^{0,1}(K)} \geq 1
\]
if one of the hyperplanes $\{x: w_{1/2}^Tx=0\}$ intersects the interior of $K$. This is the case for uncountably many $w\in S^d$. On the other hand, $L^1(\bar\pi)$ (and also its quotient by the kernel of the realization map) is separable for any Radon measure. Thus the two spaces cannot coincide. \qedsymbol

The claim can be phrased and proved in greater generality if $K$ is a manifold or similar. We note that for fixed $\bar\pi$, the space $Y_{\bar\pi}$ embeds continuously into $C^{0,1}(K)$, but its unit ball is not closed in $C^0(K)$. Nevertheless, we may consider the space
\[
\B_{Y_{\bar\pi}, K} = \left\{f_\mu(x) = \int_{B^X\cap Y_{\bar\pi, K}}\sigma(g(x))\,\mu(\d g)\:\bigg|\:\mu \text{ admissible}\right\}
\]
where admissible measures are finite (signed) Radon measures for which $Y_{\bar\pi}$ is measurable. Every distribution $\bar\pi$ on $S^d$ can be obtained as the push-forward of Lebesgue measure on the unit interval along a measurable map $\phi:(0,1)\to S^d$, see e.g.\ \cite[Section 2.8]{barron_new}. Thus the associated space of neural networks with two hidden layers is
\[
X_{(0,1), (0,1), \{0,\dots,d\}; K} = \bigcup_{\bar\pi} \B_{Y_{\bar\pi},K} =: \widetilde \W^2(K)
\]
where the union is over all probability distributions $\bar\pi$ on $S^d$. Thus the first layer of $f\in \W^2$ is wide enough to contain the entire unit ball of $\W^1$, while the first layer of $f\in \widetilde \W^2$ can only express a separable subset of the unit ball in $\W^1$. The question whether this reduces expressivity or whether in fact $\W^2 = \widetilde \W^2$ remains open.

Finally, consider the space $\W_{\L^1,\L^1,\pi^0}(K)$ where the weights of a function satisfy
\[
a^{2} \in L^2(0,1), \quad a^1 \in L^2\big((0,1)^2\big), \quad a^0 \in L^2\big((0,1)\times \{0,\dots,d\}\big) = L^2\big((0,1); \R^d).
\]
We proceed as before, but normalize with respect to $L^2$ rather than $L^1/L^\infty$. 
Again, we can consider the maps
\[
\Psi: (0,1)\times(0,1)\to \R^{d+2}, \qquad (\theta_2, \theta_1)\mapsto \big(a^1_{\theta_2\theta_1}, a^0_{\theta_1}\big)
\]
and note as before that 
\[
\int_0^1 a^1_{\theta_2, \theta_1}\,\sigma\left(\sum_{\theta_0=1}^{d+1}a^0_{\theta_1,\theta_0}x_{\theta_0}\right)\,\d\theta_1 = \int_{\R^{d+1}} \bar a^{\theta_2}(w)\,\sigma(w^Tx)\,\bar\pi(\d w)
\]
where this time $\bar a\in L^2(\bar\pi)$ for almost all $\theta_2\in(0,1)$. Thus the first layer of $f\in \W_{\L^1,\L^1,\pi^0}$ takes values in a single reproducing kernel Hilbert space $\mathcal \H_{\bar\pi}$ associated to the kernel
\[
k_{\bar\pi}(x,x') = \int_{\R^{d+1}} \,\sigma(w^Tx)\,\sigma(w^Tx')\,\bar\pi(\d w)
\]
while the first layer of $f\in \W^2$ may be wide enough to contain every function in the unit ball of Barron space. Again, the relationship between the function spaces remains open.

\subsection{Natural index sets}\label{section natural index set}

In this section, we focus on the natural index set for $\W_{\pi^L,\dots,\pi^0}(K)$. Above, we allowed the index spaces $\Omega_i$ to be generic or focused on the case $\Omega_i = (0,1)$. While $(0,1)$ is simple and mathematically convenient, it is not a natural choice. First consider the simpler case of neural networks with a single hidden layer. The classical representation in this case is
\[
f(x) = \int_{\R\times \R^{d+1}} a\,\sigma(w^Tx)\,\pi(\d a \otimes \d w)
\]
for some distribution $\pi$ on $\R^{d+2}$, see \cite{barron_new} and the sources cited therein. Using the scaling invariance $\sigma(\cdot) = \lambda^{-1}\sigma(\lambda\cdot)$ if necessary, we may assume that 
\[
\int_{\R^{d+2}} |a|^2 + |w|^2\,\pi(\d a\otimes \d w) < \infty.
\]
Then we set $\Omega_1 = \R^{d+2}, \Omega_0 = \{0,\dots, d\}$ and 
\[
a^1_{\theta_1} = (\theta_1)_1, \qquad a^0_{\theta_1,\theta_0} = (\theta_1)_{1+\theta_0},
\]
i.e.\ we index $\R^{d+2}$ by itself. In this equation, $(\theta_1)_i$ denotes the $i$-the component of the vector $\theta_1\in \R^{d+2}$. 

For networks with more than one hidden layer, the output of the first layer is vector-valued. The preceding analysis determined that the first hidden layer takes values in the reproducing kernel Hilbert space $\H_{\bar\pi}$. It thus seems reasonable at first glance to choose $\H_{\bar\pi}$ as an index space for the second hidden layer. This intuition is flawed since the output of the first hidden layer is an RKHS function of $x$, a variable which is fixed when calculating the output of the network and inaccessible to the second hidden layer. The previous observation  has no bearing on the inner workings of neural networks, but only on the approximation power of functions described by a given neural network architecture. 

Pursuing a different route, we note that $\pi$ is a Radon measure on $\R\times \R^{d+1}$ where $\R$ is the output and $\R^{d+1}$ the input layer (interpreting $x$ as $(x,1)$). For networks with two hidden layers, we note that
\begin{align*}
\bigg\| \int_{\Omega_1}&a^1_{\theta_2\theta_1}\,\sigma\left(\int_{\Omega_0}a^0_{\theta_1\theta_0}x_{\theta_0}\pi^0(\d\theta_0)\right)\pi^1(\d\theta_1)\bigg\|_{L^2(\pi_2)}^2\\
	&\leq \int_{\Omega_2} \left(\int_{\Omega_1}\big|a^1_{\theta_2\theta_1}\big|^2\,\pi^1(\d\theta_1)\right)^\frac22\left(\int_{\Omega_1}\int_{\Omega_0} \big|a^0_{\theta_1\theta_0}\big|^2\,\pi^0(\d\theta_0)\,\pi^1(\d\theta_1)\right)^\frac22 \pi^2(\d\theta_2)\sup_{x\in K}|x|^2 \\
	&= \|a^1\|_{L^2(\pi^2\otimes\pi^1)}\,\|a^0\|_{L^2(\pi^1\otimes\pi^0)}\sup_{x\in K}|x|^2
\end{align*}
for all $x\in K$. We can thus view a neural network with two hidden layers and parameter functions $a^2, a^1, a^0$ as a composition of linear and non-linear maps in the following way:

\begin{enumerate}
\item Let $\pi^1$ be the distribution of vectors $w:= (a^0_{\theta_1\theta_0})_{\theta_0=1}^{d+1}$ on $\R^{d+1}$ and $A^1:\R^d\to L^2(\pi^1)$ is the affine map described by 
\[
(A^1x)_{\theta_1} = \int_{\Omega_0} a^0_{\theta_1\theta_0}x_{\theta_0} = \frac1{d+1} \sum_{\theta_0=1}^{d+1} a^0_{\theta_1\theta_0} x_{\theta_0}.
\]
We may use $\R^{d+1}$ as its own index set, i.e.\ $a^0_{\theta_1:} = \theta_1$. To emphasize the fact that index set and distribution are natural, we denote $w=\frac{1}{d+1}\theta_1, \bar\pi = \pi^1$.

\item The non-linearity $\sigma$ acts on $L^2(\pi^1)$ by pointwise application. 

\item Let $(\Omega_2,\A_2,\pi^2)$ be a general probability space used as an index set. The linear map $A^2: L^2(\pi^1)\to L^2(\pi^2)$ is given by
\[
(A^2f)_{\theta_2} = \int_{\Omega_1} a^1_{\theta_2\theta_1}\,f_{\theta_1}\,\pi^1(\d\theta_1) = \langle a^1_{\theta_2:},z\rangle_{L^2(\pi^1)}
\]
where $a^1_{\theta_2:}(\theta_1) = a^1_{\theta_2\theta_1}$. 

\item The non-linearity $\sigma$ acts on $L^2(\pi^2)$ by pointwise application.

\item The map $A^3:L^2(\pi^2)\to \R$ is given by
\[
A^3f = \int_{\Omega_2} a^2_{\theta_2}\,f_{\theta_2}\,\pi^2(\d\theta_2).
\]
\end{enumerate}

Then
\begin{align*}
f(x) &= \big(A^3\circ\sigma\circ A^2\circ \sigma \circ A^1)(x)\\
	&= \int_{\Omega_2} a^2_{\theta_2} \,\sigma\left(\left\langle a^1_{\theta_2:}, \sigma\left(\frac1{d+1}\langle a^0_{\theta_1:},x\rangle_{\R^{d+1}} \right)\right\rangle_{L^2(\pi^1)} \right)\pi(\d\theta_2)\\
	&= \int_{\R\times L^2(\bar\pi)}\tilde a\,\sigma\left(\langle \tilde h, \sigma (w^Tx)\rangle_{L^2(\bar\pi)}\right)\,(H_\sharp \pi^2)(\d\tilde a\otimes \d \tilde h)
\end{align*}
where 
\[
H : \Omega_2\to \R\times L^2(\bar\pi), \qquad \theta_2 \mapsto (a^2_{\theta_2}, a^1_{\theta_2:}).
\]
Thus we may in a natural way interpret 
\begin{itemize}
\item $\Omega_0 = \{0,\dots, d\}$ with the normalized counting measure.
\item $\Omega_1 = \R^{d+1} = L^2(\Omega_0)$. $\bar\pi = \pi^1$ can be any probability distribution on $\Omega_1$ with finite second moments.
\item $\Omega_2= \R \times L^2(\bar \pi)$ and $\pi^2$ is a probability distribution with finite second moments.
\end{itemize}

More generally, we set
\begin{itemize}
\item $\Omega_0 = \{0,\dots, d\}$ with the normalized counting measure $\bar\pi^0$.
\item $\Omega_\ell = L^2(\bar\pi^{\ell-1})$ and a measure $\bar\pi^\ell$ with finite second moments on $\Omega_\ell$ for $1\leq \ell\leq L-1$.
\item $\Omega_L = \R\times  L^2(\bar\pi^{L-1})$ and a measure $\bar\pi^L$ with finite second moments on $\Omega_L$.
\end{itemize}
The outermost index space $\Omega_L$ has the additional factor $\R$ compared to $\Omega_\ell$ because both the first and the last operations in a neural network are linear. Note that $\Omega_\ell$ is a Polish space for every $\ell$ by induction.

All considerations above were for fixed $x$. As $x$ varies, a neural network with $L$ hidden layers takes the form $f(x) = (z^L\circ\dots\circ z^1)(x)$ where
\begin{enumerate}
\item $z^1\in C^{0,1}(\spt\,\P, \Omega_1)$, $z^0(w,x) = w^Tx = \E_{w_i\sim \pi_0}w_ix_i$ where we interpret $w\in \Omega_1 = L^2(\pi_0)$.
\item $x^\ell\in C^{0,1}(\spt\,\P, \Omega_{\ell+1})$ is defined by $z^\ell(y,f) = \langle f, \,\sigma(y)\rangle_{\pi^{\ell-1}}$ where $y\in \pi^{\ell-1}$ is the output of the previous layer and $f\in \pi^{\ell-1}$ is the natural index of $z^\ell$. Thus $z^{\ell}(\cdot, y) \in L^2(\pi^{\ell-1}) = \Omega_\ell$.
\item $z^L(y) = \int_{\Omega_L} \tilde a \,\sigma(\langle f,y\rangle_{\pi^{L-1}})\,\pi^L(\d\tilde a\otimes \d f)$.
\end{enumerate}

All natural index spaces above are separable Hilbert spaces and therefore isomorphic to each other (for all $\ell$ for which $\Omega_\ell$ if infinite-dimensional) and to both $L^2(0,1)$ and $\ell^2$. However, the application of the non-linearity $\sigma$ in $L^2$ and $\ell^2$ is not invariant under Hilbert-space isomorphisms. It makes a big difference whether we take the positive part of a function $f\in L^2(0,1)$ set all negative Fourier-coefficients of a function to zero. Luckily, natural isomorphisms preserve the structure of continuous neural network models as in Remark \ref{remark measure isomorphism}.

\section{Optimization of the continuous network model}\label{section mean-field}

We now study gradient flows for the risk functionals in the continuous setting.
We will restrict ourselves to the indexed representation with $L^2$-weights. 
The most natural optimization algorithm for weight-functions $a^\ell \in L^2((0,1)^2)$ is the $L^2$-gradient flow. We show that the usual gradient descent dynamics of neural network training can be recovered as discretizations of the continuous optimization algorithm. In this sense, we follow the philosophy of designing optimization algorithms for continuous models and discretizing them later which was put forth in \cite{E:2019aa}. We present our findings in the simplest possible setting.

\subsection{Discretizations of the  continuous gradient flow}

We now show that a natural discretization of the continuous gradient flow recovers the gradient descent dynamics for the usual
multi-layer neural networks with the ``mean-field'' scaling.
This is a general feature of Vlasov type dynamics.

The following computations are purely formal, assuming that solutions to all ODEs proposed below exist -- the issue of existence and uniqueness of solutions is briefly discussed in Appendix \ref{appendix gradient flow}.
The arguments however are based on an identity and energy dissipation property which are expected to be stable when considering generalized solutions. For smooth activation functions $\sigma$, all computations can be made rigorous and solutions exist.

\begin{lemma}
Consider a discretized version of the continuous indexed representation:
\[
f(x) = \frac1{m_L}\sum_{i_L=1}^{m_L}a^L_{i_L}\sigma\left(\frac1{m_{L-1}}\sum_{i_{L-1}=1}^{m_{L-1}} a^{L-1}_{i_Li_{L-1}}\sigma\left( \dots \sigma\left(\frac1{m_1}\sum_{i_1=1}^{m_1}a_{i_2i_1}^1\,\sigma\left(\frac1{d+1}\sum_{i_0=1}^{d+1} a^0_{i_1i_0}\,x_{i_0}\right)\right)\right)\right).
\]
Define functions
\begin{align*}
a^L:&(0,1)\to \R, \:\quad \:a^L(s) = a^L_i \:\:\:\quad\text{ for } \frac{i-1}{m_L}\leq s < \frac i{m_L}\\
a^\ell:&(0,1)^2\to \R,\quad a^\ell(r,s) = a^\ell_{ij}\quad\text{ for }\frac{i-1}{m_{\ell+1}}\leq r < \frac i{m_L}, \:\:\frac{j-1}{m_\ell} \leq s < \frac{i}{m_\ell}.
\end{align*}
for $0\leq \ell < L$. Then $f= f_{a^L,\dots, a^0}$ and the coefficient functions $a^L,\dots,a^0$ evolve by the $L^2$-gradient flow of
\[
\Risk(a^L,\dots, a^0) = \int_{\R^d} \ell\big(f_{a^L,\dots, a^0}(x),y\big)\,\P(\d x\otimes \d y)
\]
if and only if the parameters $a^L_i, a^\ell_{ij}$ evolve by the time-rescaled gradient flows
\begin{align*}
\dot a^L_i &= -m_L\,\partial_{a^L_i} \Risk\big(a^L_{i_L}, \dots, a^0_{i_1i_0}\big)\\ \showlabel
\dot a^\ell_{ij} &= - m_{\ell+1}m_\ell\,\partial_{a^\ell_{ij}} \Risk\big(a^L_{i_L}, \dots, a^0_{i_1i_0}\big) &0\leq i \leq L-1
\end{align*}
where the risk of finitely many weights is defined accordingly.
\end{lemma}

Passing to a single index set $(0,1)$ for all layers, we lose the information about the scaling of the width and compensate by prescribing layer-wise learning rates which lead to balanced training velocities. 

\begin{proof}
The proof for networks with one hidden layer can be found in \cite[Lemma 2.8]{approximationarticle}. To simplify the presentation, we focus on the case of two hidden layers. The general case follows the same way. Consider the network
\[
f(x) = \frac1M \sum_{i=1}^M a_i \,\sigma\left(\frac1m \sum_{j=1}^m b_{ij}\,\sigma\left(\frac1{d+1}\sum_{k=1}^{d+1} c_{jk}x_k\right)\right)
\]
and compute the gradient
\begin{align*}
\nabla_{a_i, b_{ij}, c_{jk}} &\Risk(a, b, c) = \nabla_{a_i, b_{ij}, c_{jk}} \int_{\R^d} \ell\big(f_{a,b,c}(x), y\big)\,\P(\d x\otimes \d y)\\
	&= \int_{\R^d} (\partial_1\ell)\big(f_{a,b,c}(x), y\big)\,\nabla_{a_i, b_{ij}, c_{jk}} f_{a,b,c,}(x)\,\P(\d x\otimes \d y)\\
	&= \int_{\R^d} (\partial_1\ell)\big(f_{a,b,c}(x), y\big)\,
		\begin{pmatrix}
		\frac1M \sigma\left(\frac1m \sum_{j=1}^m b_{ij}\,\sigma\left(\frac1{d+1}\sum_{k=1}^{d+1} c_{jk}x_k\right)\right)\\
		\frac1{M}a_i\,\sigma'\left(\frac1m \sum_{l} b_{il}\,\sigma\left(\dots\right)\right)\,\frac1m\sigma \left( \frac1{d+1}\sum_{k=1}^{d+1} c_{ik}x_k\right)\\
		\frac1M \sum_{i=1}^M a_i\,\sigma'(\dots)\,\frac1m\,\sigma'\left(\frac1{d+1}\sum_{l=1}^{d+1}c_{jl}x_l\right)\,\frac1{d+1}\,x_k
		\end{pmatrix} \,\P(\d x\otimes \d y)\\
	&= \int_{\R^d} (\partial_1\ell)\big(f_{a,b,c}(x), y\big)\,
		\begin{pmatrix}
		\frac1M \sigma\left(f_{b_{i:},c}(x)\right)\\
		\frac1{Mm}a_i\,\sigma'\left(f_{b_{i:},c}(x) \right)\,\sigma \left( f_{c_{j:}}(x)\right)\\
		\frac{1}{m(d+1)}\,\frac1M \sum_{i=1}^M a_i\,\sigma'(f_{b_{i:},c}(x))\,\,\sigma'\left(f_{c_{j:}}(x)\right)\,
	\end{pmatrix} \,\P(\d x\otimes \d y)
\end{align*}
where
\[
f_{b_{i:},c}(x) = \frac1m \sum_{j=1}^m b_{ij}\,\sigma\left(\frac1{d+1}\sum_{k=1}^{d+1} c_{jk}x_k\right) \quad\text{and}\quad f_{c_{j:}}(x) = \frac1{d+1}\sum_{l=1}^{d+1}c_{jl}x_l.
\]
Equally, we can compute the $L^2$-gradient by taking variations
\begin{align*}
\delta_{a;\phi}\Risk(a,b,c) &= \lim_{h\to 0} \frac{\Risk(a+h\phi, b, c) - \Risk(a,b,c)}h\\ \showlabel
	&= \int_{\R^d}\lim_{h\to 0} \frac{\ell\big(f_{a+h\phi, b, c}(x), y\big) - \ell\big(f_{a,b,c}(x),y\big)}h\,\P(\d x\otimes \d y)\\
	&= \int_{\R^d} (\partial_1\ell)\big(f_{a,b,c}(x), y\big)\, \lim_{h\to 0} \frac{f_{a+h\phi, b, c}(x) - f_{a,b,c}(x)}h\,\P(\d x\otimes \d y)\\
	&= \int_{\R^d} (\partial_1\ell)\big(f_{a,b,c}(x), y\big)\, \int_0^1 \phi(s)\,\sigma\big(f_{b_{s:},c}(x)\big)\ds\,\P(\d x\otimes \d y)\\
	&= \int_0^1\left(\int_{\R^d} (\partial_1\ell)\big(f_{a,b,c}(x), y\big)\,  \,\sigma\big(f_{b_{s:},c}(x)\big)\,\P(\d x\otimes \d y)\right)\phi(s)\ds
\end{align*}
since $f_{a,b,c}$ is linear in $a$. Thus the $L^2$-gradient is of $\Risk$ with respect to $a$ is represented by the $L^2$-function
\[
\delta_a\Risk(a,b,c; s) = \int_{\R^d} (\partial_1\ell)\big(f_{a,b,c}(x), y\big)\,\sigma\big(f_{b_{s:},c}(x)\big)\,\P(\d x\otimes \d y)
\]
where again
\[
f_{b_{s:},c}(x) = \int_0^1 b_{st}\left(\frac1{d+1}\sum_{i=1}^{d+1} c_{ti}x_i\right)\dt.
\]
Using the chain rule instead of linearity, we compute
\begin{align*}
\delta_{b;\phi}&\Risk(a,b,c) = \int_{\R^d} (\partial_1\ell)\big(f_{a,b,c}(x), y\big)\, \lim_{h\to 0} \frac{f_{a+h\phi, b, c}(x) - f_{a,b,c}(x)}h\,\P(\d x\otimes \d y)\\
	&= \int_{\R^d} (\partial_1\ell)(\dots) \int_0^1 a_s\,\lim_{h\to 0}\frac{\sigma\left(\int_0^1\big(b_{s,t}+h\phi_{s,t}\big)\,\sigma(f_{c_{t:}}(x))\dt\right) - \sigma\left(\int_0^1b_{s,t}\,\sigma(f_{c_{t:}}(x))\dt\right)}h\ds\,\P(\d x\otimes \d y)\\
	&= \int_{\R^d} (\partial_1\ell)(\dots) \int_0^1 a_s\,\sigma'(f_{b_{s:}c}(x))\int_0^1 \phi_{s,t}\,\sigma (f_{c_{t:}}(x))\ds \dt \,\P(\d x\otimes \d y)\\
	&= \int_{(0,1)^2}\phi_{s,t}\left(\int_{\R^d} (\partial_1\ell)(\dots) \int_0^1 a_s\,\sigma'(f_{b_{s:}c}(x))\,\sigma (f_{c_{t:}}(x))\right)\ds\dt
\end{align*}
and obtain
\begin{align*}
\delta_b\Risk(a,b,c; s, t) &= \int_{\R^d} (\partial_1\ell)\big(f_{a,b,c}(x), y\big)\,a_s \sigma'\big(f_{b_{s:}c}(x)\big) \,\sigma\left(f_{c_{t:}}(x)\right)\,\P(\d x\otimes \d y) \\
\delta_{c}\Risk(a,b,c; t) &= \int_{\R^d} (\partial_1\ell)\big(f_{a,b,c}(x), y\big)\int_0^1 a(s)\,\sigma'(f_{b_{s:}c}(x))\,b(s,t)\,\sigma'(f_{c_{t:}}(x))\,\frac{x_i}{d+1}\ds\,\P(\d x\otimes \d y).
\end{align*}
We can now see by comparing the terms that the gradient flow of a finite number of weights, interpreted as a step function, is a solution to the $L^2$-gradient flow under the appropriate time-scaling.

The general case for deep neural networks follows the same way, in which case
\begin{align*}
\delta_{a^\ell}&\Risk(a^L,\dots,a^0; \theta_{\ell+1},\theta_\ell) 
	= \int_{\R^d} (\partial_1\ell)\big(f_{a^L,\dots, a^0}(x), y\big)\int_{(0,1)^{L-\ell-1}}a^L_{\theta_L}\,\sigma'(f_{a^{L-1}_{\theta_L:}\dots a^0}(x))\,\dots a^{\ell+1}_{\theta_{\ell+1}\theta_\ell}\\
	&\hspace{5cm}\sigma'(f_{a^{\ell}_{\theta_{\ell}:}\dots a^0}(x)) \sigma(f_{a^{\ell-1}_{\theta_\ell:}\dots a^0}(x))\,\d_{\theta_L}\dots\d\theta_{\ell+2}\,\P(\d x\otimes \d y).
\end{align*}
\end{proof}

If the learning rates are not adapted to the layer width, the weights of different layers may move at different rates. In the natural time scaling, some layers would evolve at positive speed while others would remain frozen at their initial position in the limit. In particular, if the width of the two outermost layers goes to infinity, the index set of the second layer has size $m_Lm_{L-1}\gg m_L$, meaning that the outermost layer would move much faster. In \cite{araujo2019mean}, the authors consider the opposite extreme where the coefficients of the first and last layers are frozen and only intermediate layers evolve (with $m_\ell \equiv m$ for all $\ell$).

\begin{remark}
Alternative proposals for multi-layer network training in mean field 
scaling \cite{araujo2019mean,nguyen2019mean,nguyen2020rigorous,sirignano2019mean}. In this article, we opted for a particularly simple description of wide multi-layer networks and the natural extension of gradient descent dynamics. All results proved here hold for networks with finite layers of any width and therefore should remain valid more generally for another description of the parameter distribution associated to infinitely wide multi-layer networks.
\end{remark}

\subsection{Growth of the path norm}

Assuming existence of the gradient-flow evolution for the moment, we prove that the path-norm of an arbitrarily wide neural network increases at most polynomially in time under natural training dynamics. First, we consider the second moments.

\begin{lemma}\label{lemma moment growth}
Consider the risk functional
\[
\Risk(a^L,\dots a^0) = \int_{\R^d} \ell \big(f_{a^L,\dots, a^0}(x), y\big)\,\P(\d x\otimes \d y)
\]
where $\ell:\R\times\R\to[0,\infty)$ is a sufficiently smooth loss function and $\P$ is a compactly supported data distribution. Then
\begin{equation}
\left\| a^i(t)\right\|_{L^2(\pi^{i+1}\otimes\pi^i)} \leq \left\| a^i(0)\right\|_{L^2(\pi^{i+1}\otimes\pi^i)} + \sqrt{\Risk\big(a^L(0),\dots a^0(0)\big)}\,t^{1/2}
\end{equation}
\end{lemma}

\begin{proof}
We calculate 
\begin{align*}
\frac{d}{dt} \int_{\Omega_{i+1}\times \Omega_i}& \big(a^i_{\theta_{i+1}\theta_i}(t)\big)^2\,(\pi^{i+1}\otimes \pi^i) (\d\theta_{i+1}\otimes \d\theta_i) 
	= 2\int_{\Omega_{i+1}\times \Omega_i} a^i_{\theta_{i+1}\theta_i}(t) \, \frac{d\,a^i_{\theta_{i+1}\theta_i}(t)}{dt} \,\d\theta_{i+1}\, \d\theta_i\\
	&\leq 2\left(\int_{\Omega_{i+1}\times \Omega_i} \big(a^i_{\theta_{i+1}\theta_i}(t)\big)^2\,\d\theta_{i+1}\, \d\theta_i \right)^\frac12
		\left(\int_{\Omega_{i+1}\times \Omega_i} \left(\frac d{dt}a^i_{\theta_{i+1}\theta_i}(t)\right)^2\,\d\theta_{i+1}\, \d\theta_i \right)^\frac12
\end{align*}
so 
\[
\frac{d}{dt} \|a_i\|_{L^2(\pi^{i+1}\otimes\pi^i)} = \frac{\frac{d}{dt} \|a_i\|_{L^2(\pi^{i+1}\otimes\pi^i)}^2} {2\,\|a_i\|_{L^2(\pi^{i+1}\otimes\pi^i)} } \leq \left\| \frac{d}{dt}\,a^i\right\|_{L^2(\pi^{i+1}\otimes\pi^i)} \leq \left|\frac{d}{dt} \Risk(a^L,\dots a^0)\right|^\frac12
\]
since the $L^2$-gradient flow naturally satisfies the energy dissipation identity
\[
\frac{d}{dt} \Risk(a^L,\dots a^0) = -  \sum_{i=0}^L \left\| \frac{d}{dt}\,a^i\right\|_{L^2(\pi^{i+1}\otimes\pi^i)}^2.
\]
Thus 
\begin{align*}
\left\| a^i(t)\right\|_{L^2(\pi^{i+1}\otimes\pi^i)}&\leq \left\| a^i(0)\right\|_{L^2(\pi^{i+1}\otimes\pi^i)} + \int_0^t\frac{d}{ds} \|a_i(s)\|_{L^2(\pi^{i+1}\otimes\pi^i)} \,\d s\\
	&\leq \left\| a^i(0)\right\|_{L^2(\pi^{i+1}\otimes\pi^i)} + \left(\int_0^t 1 \,\d s\right)^\frac12\left(\int_0^t \left|\frac{d}{ds} \Risk\big(a^L(s),\dots a^0(s)\big)\right|\,\d s\right)^\frac12\\
	&\leq \left\| a^i(0)\right\|_{L^2(\pi^{i+1}\otimes\pi^i)} + \sqrt{\Risk\big(a^L(0),\dots a^0(0)\big)}\,t^{1/2}
\end{align*}
since the risk is monotone decreasing and bounded from below by zero.
\end{proof}

\begin{remark}
Like in \cite[Lemma 3.3]{relutraining}, a more careful analysis shows that the increase in the $L^2$-norm actually satisfies the stronger estimate
\[
\lim_{t\to \infty} \frac{\|a^i(t)\|_{L^2}}{t^{1/2}}=0.
\]
The proof of this result is based on the energy dissipation identity which characterizes weak solutions to gradient flows.
\end{remark}

\begin{corollary}\label{corollary norm growth}
Assume that $\|a^i(0)\|_{L^2(\pi^{i+1}\otimes \pi^i)} \leq C_0$ for all $i= 0,\dots, L$ and some constant $C_0>0$. Then
\begin{equation}
\|f_{a^L(t),\dots, a^0(t)}\|_{\Omega_L,\dots,\Omega_0; K} \leq \left(C_0 + \sqrt{\Risk\big(a^L(0),\dots a^0(0)\big)}\,t^{1/2}\right)^{L+1}
\end{equation}
for all $t>0$.
\end{corollary}

\begin{proof}
Follows from Lemmas \ref{lemma path-norm bound} and \ref{lemma moment growth}.
\end{proof}

As such, neural tree spaces are also the relevant class of function spaces for suitably initialized neural networks which are trained by a gradient descent algorithm. Like in \cite[Theorem 2]{dynamic_cod}, the slow increase of the norm together with the poor approximation property from Corollary \ref{corollary kolmogorov width decay} implies that the training of multi-layer networks may be subject to the curse of dimensionality when trying to approximate general Lipschitz functions in $L^2(\P)$ for a truly high-dimensional data-distribution $\P$.

\begin{corollary}\label{corollary dynamic CoD}
Consider population and empirical risk functionals 
\[
\Risk(a^L,\dots, a^0) = \frac12 \int_{[0,1]^d} (f_{a^L,\dots,a^0}- f^*)^2(x)\dx, \qquad \Risk_n(a^L,\dots, a^0) = \frac1{2n} \sum_{i=1}^n   (f_\pi- f^*)^2(x_i)
\]
where $f^*$ is a Lipschitz-continuous target function and the points $x_i$ are iid samples from the uniform distribution on $[0,1]^d$. 
There exists $f^*$ satisfying
\[
\sup_{x\in[0,1]^d} \big|f^*(x)\big| \:+\: \sup_{x\neq y} \frac{|f^*(x) - f^*(y)|}{|x-y|} \:\leq \:1
\]
such that the weight functions of $a^L,\dots,a^0$ evolving by $L^2$-gradient flow of either $\Risk_n$ or $\Risk$ satisfy
\[
\limsup_{t\to \infty} \big[t^\gamma\,\Risk(a^L(t),\dots, a^0(t))\big] = \infty
\]
for all $\gamma> \frac{2L}{d-2}$. 
\end{corollary}

\section{Conclusion}\label{section conclusion}

The classical function spaces which have been proved very successful in low-dimensional analysis (Sobolev, BV, BD, \dots) seem ill-equipped to tackle problems in machine learning. The situation has been partially remedied in some cases by introducing the function spaces associated to different models, like reproducing kernel Hilbert spaces for random feature models, Barron space for two-layer neural networks or the
flow-induced function space for infinitely deep ResNets \cite{weinan2019lei}. 

In this article, we introduced several function classes for fully connected multi-layer feed-forward networks:

\begin{enumerate}
\item The neural tree spaces $\W^L(K)$ for questions related to approximation theory and variational analysis,
\item the classes of arbitrarily wide neural networks modelled on general index spaces $\Omega_L,\dots,\Omega_0$, which we denoted by $X_{\Omega_L,\dots,\Omega_0; K}$, and
\item the classes of arbitrarily wide neural networks modelled on general index spaces with Hilbert weights (or multi-layer spaces), which we denoted by $\W_{\pi^L,\dots,\pi^0}(K)$.
\end{enumerate}

The key to the definition of these spaces is the representation of functions. 

Neural tree spaces are built using a tree-like index structure, and network weights have no natural meaning. This point of view thus cannot encompass training algorithms which operate on network weights. By analogy with classical approximation theory, we can think of finite neural networks as polynomials (finitely parametrized functions) and of neural tree spaces as Sobolev or Besov classes obtained as the closure under a weak norm, but too general for classical Taylor series. We denoted these by $\W$ for `wide' structures.

The classes of arbitrarily wide neural networks are introduced as very general function classes which exhibit the natural neural network structure via generalized index spaces. In the general class of arbitrarily wide networks, weight functions are assumed to be merely measurable with integrable products, which is a too large space to study training dynamics. The restriction of the multi-layer norm to this space is a natural norm, and the closure of the unit ball in the space of arbitrarily wide neural networks and neural tree space coincides.

To study training dynamics, we consider the space of arbitrarily wide neural networks with Hilbert weights, where the $L^2$-inner product induces a gradient flow in the natural way. The restriction of the path norm does not control the $L^2$-magnitude of the weight functions, so we studied a different measure of complexity on this function space (which is not usually a norm). The complexity measure was seen to bound the path-norm from above and to grow at most like $t^\frac{L+1}2$ in time under gradient flow training. 

It is immediate that $\W_{\pi^L,\dots,\pi^0}(K) \subseteq X_{\Omega_L,\dots,\Omega_0; K} \subseteq \W^L(K)$ with inclusions that are strict if the index spaces are finite. Whether the inclusions are strict in the general case, is not clear. In the case of three-layer networks, they can be interpreted as the spaces in which the first hidden layer is wide enough to output Barron space, a separable subspace of Barron space and a reproducing kernel Hilbert space respectively. All three spaces contain all Barron functions and their compositions.

One naturally asks which one of these spaces is most suited for describing multi-layer neural networks.
An ideal space should  (1) be complete, (2) have a nice approximation theory, (3) have a low Rademacher complexity, and
(4) most importantly, be concrete enough so that one can make use of the function representation for practical purposes.
At this point, we cannot prove any of the spaces introduced here  satisfies all these requirements.
Our feeling is that the space $\W_{\pi^L,\dots,\pi^0}(K)$ for sufficiently large index spaces $(\Omega_\ell, \A_\ell, \pi^\ell)$ might be the most promising one, even though at this point it is only a metric vector space, not a normed space (see Definitions \ref{definition metric hilbert weights} and \ref{definition multi-layer space new} and the surrounding paragraphs.).  However, it  seems to be the most relevant space for practical purposes.

A number of questions remain open.

\begin{enumerate}
\item Beyond first observations, the relationship between the neural tree spaces $\W^L(K)$ and its subspace $\W_{\pi^L,\dots,\pi^0}(K)$ for sufficiently expressive index sets remains unexplored. The first space is suited for variational and approximation problems, while the second is a natural object for mean-field training. It is an important question how much of the hypothesis space we can explore using natural training dynamics.

Even for networks with two hidden layers, only heuristic observations about $\W^2(K)$ and its subspaces $\W_{\L,\L,\pi^0;K}$ and $\widetilde \W^2 = X_{(0,1), (0,1), \{0,\dots, d\}; K}$ of network-like functions are available. Whether the two can be treated in a unified perspective remains to be seen. 

\item The direct approximation theorem holds for neural tree spaces, but not with the Monte-Carlo rate(in terms of free parameters). Whether a better rate can be achieved for functions in neural tree space for $L\gg1$ (or at least a space of arbitrarily wide neural networks) remains an important open problem.

\item The properties of the complete metric vector spaces $\W_{\pi^L,\dots\pi^0}(K)$ have not been studied yet.

\item We defined a monotonically increasing sequence of spaces $\W^L$ for $L\in \N$. Examples \ref{example barron barron} and \ref{example barron lipschitz} show that $\B_{X,K}$ may much larger than $X$ or exactly the same, depending on $X$. Concerning neural networks, it is clear that $\W^1$ is much larger than $\W^0$. In \cite{barron_new}, we give give an easy to check criterion which implies that a function is not in $\W^1$ and provide examples of functions which are in $\W_{\L,\L,\pi^0}(K) \subseteq\W^2$, but not $\W^1$. Beyond this, the relationship between the spaces $\W^\ell$ and $\W^L$ for $\ell < L$ is largely unexplored.

\item In this paper,  we considered the minimization of an integral risk functional. A more classical problem in numerical analysis concerns the discretization of variational problems and partial differential equations. In both applications, a key component is the approximation of a solution $f^*$ of the problem by functions $f_m$ in a finitely parameterized hypothesis class (Galerkin spaces or neural networks). Often, the approximation rate $\|f_m - f^*\| \leq m^{-\alpha}$ of solutions $f_m$ of the discretized problem to the true solution depends on the properties of $f^*$ (as well as the choice of norm).

For many variational problems and partial differential equations, a priori estimates on the solutions in Sobolev or H\"older spaces are available. The regularity of $f^*$ is therefore understood, as well as the expected rate of convergence $f_m\to f$. 

In machine learning, a regularity theory of this type is generally missing. It is often unclear in which function space the minimizer of a well-posed risk functional should lie, and thus equally unclear what type of machine learning model to use (random feature model, shallow neural network, deep neural network, ResNet, \dots). A regularity theory which bounds the necessary number of layers in a neural network from above or below even for specific learning applications is not yet available.

As shown in Corollary \ref{corollary dynamic CoD}, gradient descent may converge very slowly if the target function does not lie in the correct target space and $\frac Ld\ll 1$.

\item Even assuming that the solution to a variational problem is known explicitly, it remains difficult to decide whether it lies in $\W^L$ for a given $L$. Only for $L=1$ a positive criterion is given in \cite{barron1993universal} and a negative criterion following \cite[Theorem 5.4]{barron_new}. In general, it remains hard to check whether a function can be expressed as a neural network of depth $L$.

\item In this article, we focused on fully connected networks with infinitely wide layers. The theory for other types of neural networks (convolutional, recurrent, residual) will be the subject of future work.

Starting with the articles \cite{haber2017stable,weinan2017proposal, li2017maximum, E:2018aa}, deep ResNets have been modeled as discretizations of an ODE flow (sometimes referred to as `neural ODEs'). A function space for infinitely deep residual networks with skip-connections after {\em every} layer has been proposed in \cite{weinan2019lei}. In this model, the width of incremented layers is constant, but the width of the residual block may go to infinity. The case of ResNets which are both very wide and very deep and have skip-connections every $\ell\geq1$ layers is currently unexplored.

As demonstrated in Example \ref{example rademacher complexity lower bound}, Rademacher complexity cannot give a significantly better generalization bound for the space of convolutional networks than for the space of fully connected networks. Despite many heuristic explanations, the factors contributing to the success of convolutional networks in image processing have not been understood rigorously (for non-linear activation functions).

\item Even for finite neural networks with ReLU activation and more than one hidden layer, we are not aware of rigorous results for the existence of solutions to the gradient flow equations in any strong or weak formulation.

\item In many applications, neural networks are initialized with parameters that scale in such a way that the path-norm grows beyond all bounds as the number of neurons increases. Learning rates may not be adapted to the width of the layers in applications, and the scaling invariance $\sigma(z) \equiv \lambda\,\sigma(\lambda^{-1}z)$ for $\lambda>0$ may lead to coefficients which are of very different magnitude on different layers. In this situation, our  analysis does not apply, and it can be shown rigorously in some cases that very wide networks of fixed depth may behave like linear models \cite{arora2019exact,du2018gradient,du2018bgradient,weinan2019analysis,weinan2019comparative,jacot2018neural}.

These analyses typically make use of over-parametrization by assuming that the network has many more neurons than the data set has training samples. In this scaling regime, the correct function spaces and training dynamics for wide networks under population risk are generally unexplored.

\end{enumerate}


\newcommand{\etalchar}[1]{$^{#1}$}

\appendix
\section{A brief review of measure theory}\label{appendix measure theory}

We briefly review some notions of measure theory used throughout the article. We assume familiarity with the basic notions of topology, measure theory, and functional analysis (metrics, topologies, $\sigma$-algebras, measures, Banach spaces, dual spaces, weak topologies, \dots). Further background material can be found e.g.\ in \cite{MR2759829, elstrodt1996mass, munkres1974topology, yosida2012functional, klenke2006wahrscheinlichkeitstheorie}.

\subsection{General measure theory}
Let $(X, \A)$ be a measurable space. A signed measure is a map $\mu:\A\to \R\cup\{-\infty, \infty\}$ such that for any collection $\{A_i\}_{i\in \Z}$ of measurable disjoint sets we have
\[
\mu\left(\bigcup_{i=1}^\infty A_i\right) = \sum_{i=1}^\infty \mu(A_i)
\]
($\sigma$-additivity), assuming that the right hand side is defined. A signed measure $\mu$ admits a Hahn decomposition $\mu = \mu^+ - \mu^-$ where $\mu^+, \mu^-$ are mutually singular (non-negative) measures. All proofs for this section can be found in \cite[Chapter 7.5]{klenke2006wahrscheinlichkeitstheorie} for proofs in this section. Being mutually singular means that there exist measurable sets $A_+, A_-$ such that
\[
\mu^+(A^+) = \mu^+(X), \quad \mu^-(A^+) = 0, \qquad \mu^+(A^-) = 0, \quad \mu^-(A^-) = \mu^-(X),
\]
i.e.\ $\mu_+, \mu_-$ ``live'' on different subset of $X$. The (non-negative) measure $|\mu| = \mu^+ + \mu^-$ is called the total variation measure of $\mu$. The total variation norm of $\mu$ is defined as
\[
\|\mu\| = |\mu|(X) = \mu^+(A^+) + \mu^-(A^-) = \sup_{A, A' \in \A} \mu(A) - \mu(A').
\]

Let $X, Y$ be measurable spaces, $\phi:X\to Y$ a measurable map and $\mu$ a (signed) measure on $X$. Then we define the push-forward $\phi_\sharp\mu$ of $\mu$ along $\phi$ by $(\phi_\sharp\mu)(A) = \mu(\phi^{-1}(A))$ for all measurable $A\subseteq Y$. Note that by definition
\[
\int_X f(\phi(x))\,\mu(\d x) = \int_Y f(y)\,(\phi_\sharp\mu)(\d y) \qquad \forall\ f:Y\to \R.
\]
Furthermore, $\|\phi_\sharp\mu\| \leq \|\mu\|$ (since the images $\phi(A^+)$ and $\phi(A^-)$ may intersect non-trivially) and $\|\phi_\sharp\mu\| = \|\mu\|$ if $\mu$ is a (non-negative) measure (since no cancellations can occur).

\subsection{Measure theory and topology}

All measurable spaces considered in this article have compatible topological and measure theoretic structures. The following kind of spaces have proved to be well suited for many applications.

\begin{definition}
A {\em Polish space} is a second countable topological space $X$ such that there exists a metric $d$ on $X$ which induces the topology of $X$ and such that $(X,d)$ is a complete metric space.
\end{definition}

In particular, compact metric spaces are Polish. Since Polish spaces are metrizable, being second countable and separable is equivalent here. 

\begin{lemma}\cite[Appendix A.22]{elstrodt1996mass}
Let $X,Y$ be Polish spaces. The following are Polish spaces.
\begin{enumerate}
\item An open subset $U\subseteq X$ with the subspace topology.
\item A closed subset $U\subseteq X$ with the subspace topology.
\item $X\times Y$ with the product topology.
\end{enumerate}
\end{lemma}

All but the first point are trivial. If $U$ is a non-empty open set, note that the metric
\[
d_U(x,x') = d(x,x') + \big|f_U(x) - f_U(x')\big|, \qquad f_U(x) = \frac1{\dist(x, \partial U)}
\]
induces the same topology as $d$ on $U$ and is complete if $d$ is complete on $X$.
There are various compatibility notions between the topological structure and measure theoretic structure of a space $X$.

\begin{definition}
Let $X$ be a Hausdorff space (so that compact sets are closed $\Ra$ Borel).
\begin{enumerate}
\item The Borel $\sigma$-algebra is the $\sigma$-algebra generated by the collection of open subsets of $X$. We will always assume that measures are defined on a the Borel $\sigma$-algebra.
\item A measure $\mu$ is called {\em locally finite} if every set $x\in X$ has a neighbourhood $U$ such that $\mu(U) <\infty$. Locally finite measures are also referred to as {\em Borel measures}.
\item A measure $\mu$ is called {\em inner regular} if 
\[
\mu(A) = \sup\{\mu(K) \:|\: K\subseteq A, \:K \text{ is compact}\}
\]
for all measurable sets $A$. An inner regular Borel measure is called a {\em Radon measure}.
\item A measure $\mu$ is called {\em outer regular} if 
\[
\mu(U) = \inf\{\mu(U)\:|\:A\subseteq U, \:U\text{ is open}\}
\]
for all measurable sets $A$. A measure is called {\em regular} if it is both inner and outer regular.
\item A measure $\mu$ is called {\em moderate} if $X = \bigcup_{k=1}^\infty U_k$ where the $U_k$ are open sets of finite measure.
\end{enumerate}
\end{definition}

On Polish spaces, most measures of importance are Radon measures. The following result is due to Ulam.

\begin{theorem}\cite[Kapitel VIII, Satz 1.16]{elstrodt1996mass}
Let $X$ be a Polish space. Then every Borel measure $\mu$ on $X$ is moderate and regular (in particular, a Radon measure).
\end{theorem}

For Radon measures, we can define the analogue of the support of a function to capture the set the measure `sees'.

\begin{definition}
Let $\mu$ be a Radon measure. We set
\[
\spt(\mu) = \bigcap_{K\text{ closed, }\mu(X\setminus K)=0} K.
\]
\end{definition}

The support of a measure is closed. Note that the measure $\mu = \sum_{i=1}^\infty a_i \,\delta_{q_i}$ has support $\R$ if $a_i$ is a summable sequence of positive numbers and $q_i$ is an enumeration of $\Q$. We say that $\mu$ {\em concentrates} on $\Q$ since $\mu(\R\setminus \Q) =0$. The support of a measure $\mu$ can be significantly larger than a set on which $\mu$ concentrates.

\subsection{Continuous functions on metric spaces}

In many analysis classes, the space of continuous functions on $[0,1]$ is shown to be separable as a corollary to the Stone-Weierstrass theorem with the dense set of polynomials with rational coefficients. This can be shown in a simpler way and greater generality.

\begin{theorem}
Let $X$ be a compact metric space and $C(X)$ the space of continuous real-valued functions on $X$ with the supremum norm. Then $C(X)$ is separable.
\end{theorem}

\begin{proof}
Since $X$ is compact, it has a countable dense subset $\{x_n\}_{n\in\N}$. Consider a family of continuous functions $\eta_{n, m}:X\to [0,1]$ such that
\[
\eta_{n,m}(x) = \begin{cases} 1 &d(x, x_n) \leq \frac1m \\ 0 &d(x, x_n) \geq \frac2m\end{cases}.
\]
Denote
\[
\F_{n, m} = \left\{ \sum_{i=1}^n \sum_{j=1}^m a_{i,j}\,\eta_{i,j}(x) \:\bigg|\: a_{i,j} \in \mathbb Q\:\:\forall\ i,j\in\N\right\}, \qquad \F = \bigcup_{n,m=1}^\infty \F_{n,m}.
\]
Then $\F$ is a countable subset of $C(X)$. If $f:X\to\R$ is continuous, it is uniformly continuous, and it is easy to see by contradiction that $f$ can be approximated uniformly by functions in $\F$. 
\end{proof}

\begin{remark}
The same holds for the space of continuous functions from a compact metric space $X$ into a separable metric space $Y$ with the metric
\[
d(f, g) = \sup_{x\in X} d_Y(f(x), g(x))
\]
and more generally on locally compact Hausdorff spaces and the compact-open topology on the space of continuous maps.
\end{remark}

\subsection{Measure theory and functional analysis}

Radon measures allow a convenient functional analytic interpretation due to the following Riesz representation theorem. We only invoke the theorem in the special case of compact spaces and note that compact metric spaces are both locally compact and separable. The same result holds in greater generality, which we shall avoid to focus on the setting where the space of continuous functions is a Banach space.

\begin{theorem}\cite[Theorem 1.54]{ambrosio2000functions}
Let $X$ be a compact metric space and $C(X;\R^m)$ the space of all continuous $\R^m$-valued functions on $X$. Let $L$ be a continuous linear functional on $C(X; \R^m)$. Then there exist a (non-negative) Radon measure $\mu$ and a $\mu$-measurable function $\nu:X\to S^{m-1}$ such that 
\[
L(f) = \int_X \langle f(x),\nu(x)\rangle \,\mu(\d x)\qquad \forall\ f\in C(X;\R^m).
\]
Furthermore, $\|L\|_{C(X;\R^m)^*} = \|\mu\|$.
\end{theorem}

Denote by $\mathcal A$ the Borel $\sigma$-algebra of $X$. The function 
\[
\nu\cdot\mu: \mathcal A\to \R^m, \qquad (\nu\cdot\mu)(A) = \int_A \nu(x)\,\mu(\d x)
\]
is called a {\em vector valued Radon measure} if $m\geq 2$ (and a {\em signed Radon measure} if $m=1$). Vector-valued Radon measures are $\sigma$-additive on the Borel $\sigma$-algebra. The measure $\mu$ is called the total variation measure of $\nu\cdot\mu$. In the following, we will denote vector-valued Radon measures simply by $\mu$ and the total variation measure by $|\mu|$, like we did before for signed measures. The theorem admits the following interpretation and extension.

\begin{theorem}
The dual space of $C(X;\R^m)$ is the space of $\R^m$-valued Radon measures $\M(X;\R^m)$ with the norm
\[
\|\mu\|_{\M(X;\R^m)} = |\mu|(X).
\]
We denote the space of $\R^m$-valued Radon measures by $\M(X;\R^m)$ and $\M(X;\R)=: \M(X)$.
\end{theorem}

\begin{definition}
We say that a sequence of (signed, vector-valued) Radon measures $\mu_n$ converges weakly to $\mu$ and write $\mu_n\wto\mu$ if 
\[
\int_X f(x)\,\mu_n(\d x) \to \int_X f(x)\,\mu(\d x)\qquad \forall\ f\in C(X) = C(X;\R).
\]
\end{definition}

In this terminology, the weak convergence of Radon measures coincides with weak* convergence in the dual space of $C(X)$. By the Banach-Alaoglu theorem \cite[Theorem 3.16]{MR2759829}, the unit ball of $\M(X)$ is compact in the weak* topology. Since $C(X)$ is separable, the weak* topology of $\M(X)$ is metrizable \cite[Theorem 3.28]{MR2759829}. Thus if $\mu_n$ is a bounded sequence in $\M(X)$, there exists a weakly convergent subsequence. This establishes the {\em compactness theorem for Radon measures}.

\begin{theorem}\label{compactness theorem radon measures}
Let $\mu_n$ be a sequence of (signed, vector-valued) Radon measures such that $\|\mu_n\|\leq 1$. Then there exists a (signed, vector-valued) Radon measure $\mu$ such that $\mu_n\wto \mu$.
\end{theorem}

A good exposition in the context of Euclidean spaces can be found in \cite[Chapter 1]{evans2015measure} with arguments which can be applied more generally.

\subsection{Bochner integrals}

Bochner integrals are a generalization of Lebesgue integrals to functions with values in Banach spaces. A quick introduction can be found e.g. in \cite[Chapter V, part 5]{yosida2012functional} or \cite[Kapitel 2.1]{ruzicka2006nichtlineare}.

\begin{definition}
Let $(X,\A,\mu)$ be a measure space and $Y$ a Banach space. A function $f:X\to Y$ is called {\em Bochner-measurable} if there exists a sequence of step functions $f_n= \sum_{i=1}^n y_i\chi_{A_i}$ with $y_i\in Y, A_i \in \A$ such that $f_n\to f$ pointwise $\mu$-almost everywhere. 
\end{definition}

For real-valued functions, Bochner-measurability coincides with the usual notion of measurability.

\begin{lemma}
Let $X$ be a compact metric space, $\A$ its Borel sigma algebra, $\mu$ a measure on $\A$ and $Y$ a Banach space. Then every continuous function $f:X\to Y$ is uniformly continuous and thus Bochner-measurable.
\end{lemma}

A function $f$ is {\em Bochner-integrable} if the integrals $\sum_{i=1}^n\mu(A_i)\,y_i$ of the approximating sequence $f_n$ converge and do not depend on the choice of $f_n$.

\begin{lemma}
Let $X$ be a compact metric space, $\A$ its Borel sigma algebra, $\mu$ a finite measure on $\A$ and $Y$ a Banach space. Then every continuous function $f:X\to Y$ is additionally bounded and thus Bochner-integrable.
\end{lemma}

Bochner-integrals are linked to Lebesgue-integrals in the following way.

\begin{lemma}
Let $f$ be a Bochner-measurable function. Then $f$ is Bochner-integrable if and only if $\|f\|:X\to\R$ is Lebesgue-integrable. Furthermore, 
\[
\left\|\int_X f(x)\,\mu(\d x)\right\|_Y \leq \int_X \|f(x)\|_Y\,\mu(\d x).
\]
\end{lemma}

If $\mu$ is a finite signed measure, these notions generalize in the obvious way.

\begin{definition}
Let $(\Omega,\A,\mu)$ be a measure space, $p\in [1,\infty]$ and $X$ a Banach space. Then the Bochner space $L^p(\Omega; X)$ is the space of all Bochner-measurable functions $f:\Omega\to X$ such that $\|f\| \in L^p(\Omega)$. 
\end{definition}

The following is proved in the unnumbered example following \cite[Lemma 1.23]{ruzicka2006nichtlineare}. The claim is formulated in the special case where $\Omega_1$ is an interval and $\Omega_2\subseteq\R^d$, but the proof holds more generally.

\begin{lemma}
Let $(\Omega_i,\A_i,\mu_i)$ be measure spaces for $i=1,2$. Then $f\in L^p(\mu_1\otimes \mu_2)$ if and only if the function
\[
F:\Omega_1 \to L^p(\Omega_2), \qquad \big[F(\omega_1)\big](\omega_2) = f(\omega_1,\omega_2)
\]
is well-defined and in $L^p(\Omega_1, L^p(\Omega_2))$.
\end{lemma}

Furthermore, we recall the following immediate result, which we will apply in conjunction with the previous Lemma in the special case that $H = L^2(0,1)$. 

\begin{lemma}
If $H$ is a Hilbert space, so is $L^2(\Omega; H)$ with the inner production
\[
\langle f, g\rangle_{L^2(H)} = \int_\Omega \langle f(\omega), g(\omega)\rangle\,\mu(\d \omega).
\]
\end{lemma}

\section{On the existence and uniqueness of the gradient flow}\label{appendix gradient flow}

For networks with smooth activation functions, the preceding analysis can be justified rigorously. We briefly discuss some obstacles in the case of ReLU activation.

\begin{example}
Generically, solutions of gradient flow training for ReLU-activation are non-unique, even for functions with one hidden layer. We consider a network with one hidden layer, one neuron, and a risk functional with one data point:
\[
 f_{a,b}(x) = a\,\sigma(b^1x-b^2), \qquad \Risk(a,b) = \big|f_{a,b}(1) - 1\big|^2 = \big|a(b^1-b^2)_+ - 1\big|^2.
\]
If $a,b$ is initialized as $a_0 = 1$, $b_0= (1,1)$, then one solution of the gradient flow inclusion is constant in time. This solution is obtained as the limit of gradient flow training for regularized activation functions $\sigma^\eps$ satisfying $(\sigma^\eps)'(0) = 0$. Another solution is the solution $(a,b)$ of ReLU training is
\[
\begin{pmatrix}\dot a_t\\ \dot b^1_t\\ \dot b^2_t\end{pmatrix}
	= - \nabla_{a,b}\big|a(b^1-b^2) - 1\big|^2 = - 2\big(a(b^1-b^2) - 1\big)\begin{pmatrix} b^1-b^2\\ a\\ -a\end{pmatrix},
\]
for which the risk decays to zero. This is obtained as the limit of approximating gradient flows associated to $\sigma^\eps$ with $(\sigma^\eps)'(0)=1$.
\end{example}
  
As the training dynamics are non-unique, the Picard-Lindel\"off theorem cannot apply. In \cite[Lemma 3.1]{relutraining}, we showed that the situation can be remedied by considering gradient flows of population risk for suitably regular data distributions $\P$. A key ingredient of the proof is that for fixed $w$, $\sigma'(w^Tx)$ is well-defined except on a hyper-plane in $\R^d$, which we assume to be $\P$-null sets. An existence proof based on the Peano existence theorem is also presented in a specific context in \cite{Chizat:2020aa}.

This argument cannot be extended to networks with multiple hidden layers since terms of the form $\sigma'(f(x))$ occur where $f$ can be a general Barron function (or even more general for deep networks). The level sets of Barron functions may be highly irregular and even for $C^1$-smooth Barron functions, Sard's theorem need not apply \cite[Remark 3.2]{barron_new}. In particular, for any data distribution $\P$, we can find a non-constant Barron function $f$ such that $\P(\{f=0\})>0$. It thus appears inevitable to consider a class of weak solutions based on energy dissipation properties or differential inclusions. We note that the proofs in this article are based on purely formal identities and the energy dissipation property. We thus expect the results to remain valid for suitable generalized solutions.
\end{document}